 \documentclass[10pt, final, journal, letterpaper, twocolumn]{IEEEtran}
 \usepackage{amsmath,amssymb}
 \usepackage{graphicx,graphics,color,psfrag}
 \usepackage{cite,balance}
 \usepackage{caption}
 \allowdisplaybreaks
 \usepackage{algorithm}
 \usepackage{accents}
 \usepackage{amsthm}
 \usepackage{hyperref}
 \usepackage{amsmath}
 \usepackage{bm}
 \usepackage{algorithmic}
 \usepackage[english]{babel}
 \usepackage{multirow}
 \usepackage{enumerate}
 \usepackage{cases}
 \usepackage{stfloats}
 \usepackage{dsfont}
 \usepackage{color,soul}
 \usepackage{amsfonts}
 \usepackage{cite,graphicx,amsmath,amssymb}
 \usepackage{framed}
 \usepackage{fancyhdr}
 \usepackage{hhline}
 \usepackage{ulem}
 \usepackage{graphicx,graphics}
 \usepackage{array,color}
 \usepackage{amsmath}
 \usepackage{booktabs}
 \usepackage{subfigure}
 \usepackage[nolist]{acronym}

\begin{acronym}

\acro{BNN}{Bayesian Neural Networks}
\acro{CDF}{cumulative distribution function}
\acro{DSVGD}{Distributed Stein Variational Gradient Descent}
\acro{DSGLD}{Distributed Stochastic Gradient Langevin Dynamics}
\acro{DVI}{Distributed Variational Inference}

\acro{ELBO}{Evidence Lower Bound}
\acro{FL}{Federated Learning}
\acro{FedAvg}{Federated Averaging}
\acro{FedSGD}{Federated Stochastic Gradient Descent}
\acro{GVI}{Global Variational Inference}
\acro{HIP}{Hilbert Inner Product}

\acro{KL}{Kullback–Leibler}
\acro{KDE}{Kernel Density Estimator}
\acro{KSD}{Kernelized Stein Discrepancy}
\acro{MCMC}{Markov Chain Monte Carlo}
\acro{MC}{Monte Carlo}
\acro{MCE}{Maximum Calibration Error}
\acro{NPV}{Non-Parametric Variational Inference}

\acro{PDF}{probability density function}%
\acro{PVI}{Partitioned Variational Inference}
\acro{PMD}{Particle Mirror Descent}
\acro{P-DSVGD}{Parallel-Distributed Stein Variational Gradient Descent}

\acro{RKHS}{Reproducing Kernel Hilbert Space}
\acro{SPSA}{simultaneous perturbation stochastic approximation}
\acro{RBF}{Radial Basis Function}
\acro{RMSE}{Root Mean Square Error}
\acro{SVGD}{Stein Variational Gradient Descent}
\acro{SGLD}{Stochastic Gradient Langevin Dynamics}
\acro{U-DSVGD}{Unconstrained-DSVGD}
\acro{VI}{Variational Inference}

\end{acronym}

\makeatletter
\renewcommand{\fnum@figure}{Fig. \thefigure}
\makeatother

\newtheoremstyle{itshape}
  {.0\baselineskip\@plus.0\baselineskip\@minus.0\baselineskip}
  {.0\baselineskip\@plus.0\baselineskip\@minus.0\baselineskip}
  {\itshape}
  {}
  {\bfseries}
  {.}
  { }
  {}

\theoremstyle{itshape}

\newtheorem{theorem}{Theorem}

\newtheorem{lemma}{Lemma}

\newtheorem{definition}{Definition}

\renewcommand{\algorithmicrequire}{ \textbf{Input:}}
\renewcommand{\algorithmicensure}{ \textbf{Output:}}

\begin{document}
\include{header}
\title{Client Selection for Federated Bayesian Learning}

\author{Jiarong Yang, Yuan Liu, and Rahif Kassab 

\thanks{This work was supported in part by the Natural Science Foundation of China under Grant 61971196, Grant U1701265, and Grant U2001210. \emph{(Corresponding author: Yuan Liu.)}
\par
J. Yang and Y. Liu are with school of Electronic and Information Engineering, South China University of Technology, Guangzhou 510641, China (e-mails: eejiarong@mail.scut.edu.cn, eeyliu@scut.edu.cn). 
\par
Rahif Kassab is with the King’s Communications, Learning and Information Processing (KCLIP) Lab King’s College London, London WC2B4BG, U.K. (email: rahif.kassab@kcl.ac.uk). 
}
}

\maketitle

\vspace{-1.5cm}
\begin{abstract}
Distributed Stein Variational Gradient Descent (DSVGD) is a non-parametric distributed learning framework for federated Bayesian learning, where multiple clients jointly train a machine learning model by communicating a number of non-random and interacting particles with the server. Since communication resources are limited, selecting the clients with most informative local learning updates can improve the model convergence and communication efficiency. In this paper, we propose two selection schemes for DSVGD based on Kernelized Stein Discrepancy (KSD) and Hilbert Inner Product (HIP). We derive the upper bound on the decrease of the global free energy per iteration for both schemes, which is then minimized to speed up the model convergence. We evaluate and compare our schemes with conventional schemes in terms of model accuracy, convergence speed, and stability using various learning tasks and datasets.
\end{abstract}


\section{Introduction}
The rapid development of machine learning has impacted various aspects of today's society, from autonomous vehicles to computer vision, bioscience and natural language processing \cite{Jordan255}. Conventionally, training machine learning models relies on the availability of big datasets at a centralized remote location, e.g., a central server, where training is carried out. However, most data is generated in a distributed fashion, e.g., by edge devices/clients, and thus needs to be collected at the central server to enable training. With the increasing awareness for data security and privacy, uploading local data to a remote location is no longer a viable solution. To cope with this challenge, federated learning \cite{10.1145/3298981,8940936,9084352} was proposed as a distributed learning framework that performs model training at the client side using the client's local dataset and then shares the client's model with the server, where various models are aggregated. Thus, under this framework, the client's data is preserved on their device. In some practical scenarios, such as medical diagnosis and autonomous driving, models need to provide not only predictions but also a calibrated confidence measure that reflects the amount by which the client can trust a given prediction. However, most federated learning algorithms are based on frequentist learning principle which is known to be unable to quantify epistemic uncertainty, yielding overconfident decisions \cite{pmlr-v70-guo17a,lakshminarayanan2017simple}. In contrast,  Bayesian learning is more suitable in such scenarios as it provides a more accurate estimate of uncertainty by optimizing over the space of distributions of model parameters instead of a single vector of model parameters \cite{Barber2012BayesianRA,8453245}. 

The main goal of \emph{federated Bayesian learning} is to train a global posterior by clients which exchange model parameters with a central server over a wireless channel \cite{vehtari2020expectation,bui2018partitioned,corinzia2019variational}. Therefore, the convergence of the model is often affected by the number of communication rounds  \cite{ren2020accelerating,zhao2020federated} and data-heterogeneity \cite{li2019convergence,zhao2018federated}. Specifically, on one hand, model training requires frequent exchanges of model parameters between the server and clients over a channel, so a large number of communication resources (i.e., time and frequency resources) will cause large communication delays, which in turn affects the convergence speed of the model. On the other hand, clients are isolated from each other, consequently, there is no guarantee that the data between clients is independent and identically distributed (IID). Such heterogeneity of client's data will reduce the model accuracy and slow down the convergence speed. Therefore, it is concluded that local learning updates  of distributed clients do not contribute equally to global model convergence \cite{zhao2015stochastic,liu2020data,zeng2021noise}. A promising technique to tackle this issue is scheduling or client selection, i.e., select the clients with most informative data for participation so as to not only accelerate the learning process but also reduce communication overhead.

\subsection{Prior Work}
Existing federated Bayesian learning techniques can be classified to two families: \ac{MC} sampling or \ac{VI}. \ac{MC}-based federated Bayesian learning\cite{ahn2014distributed, pmlr-v115-mesquita20a} such as \ac{DSGLD}, maintains a number of Markov chains that are updated via local Stochastic Gradient Descent (SGD) and injects noise into the parameter updates \cite{welling2011bayesian, ahn2014distributed}. Unfortunately, the model convergence for such techniques is difficult to assess and the convergence speed is low. \ac{VI}-based federated Bayesian learning \cite{angelino2016patterns, corinzia2019variational, bui2018partitioned,kassab2022federated} such as \ac{PVI} optimizes over parametric posteriors from tractable families (e.g., the exponential family) using natural gradient descent \cite{bui2018partitioned}. However, limiting the optimization to tractable families incurs bias. Accordingly, \ac{DSVGD} has been proposed as a method to eliminate such bias by encoding the variational posterior with deterministic particles that are updated by each scheduled client using \ac{SVGD}  \cite{10.5555/3157096.3157362}. When the number of particles is large enough, any posterior could be approximated. By tuning the number of particles, \ac{DSVGD} can trade off bias, convergence speed and per-iteration complexity \cite{kassab2022federated}. More details on \ac{SVGD} and \ac{DSVGD} are given in Section \ref{sec:system_model_and_learning_mechanism}.
\par
One of the biggest challenges in federated learning is decreasing the communication load. This challenge, known as \emph{communication-efficient} federated learning, recently received lots of interest in the wireless research community. More specifically, multiple techniques have been studied to decrease the communication  load, for e.g., through the joint design of learning and communication such as learning update compression \cite{sattler2019sparse, lin2018deep}, radio resource management \cite{wang2019adaptive, chen2019performance, tran2019federated} and over-the-air computation \cite{yang2020federated, zhu2019broadband, zhu2020one}. Another technique, often referred to as \emph{client selection} or scheduling, exploits the difference in the clients' model updates and selects the clients with the most informative updates in order to increase the convergence speed while decreasing the total communication load \cite{zhao2015stochastic, 9107235,9170917,9252927}. For example, authors of \cite{9170917} use gradient divergence to quantify the importance of each local update, and propose a novel importance and channel-aware scheduling policy. In \cite{9252927}, the inner product between each local gradient and the ground-truth global gradient is used to quantify the importance of data, and a fast convergent algorithm is proposed where the local update is weighted by the the inner product. 
\par
Most of the existing  schemes considered in the literature are applicable to frequentist federated learning. In this paper, we explore the benefits of client selection for  federated Bayesian learning. More specifically, we consider \ac{DSVGD} \cite{kassab2022federated}, which is a non-parametric distributed variational inference framework for federated Bayesian learning.  \ac{DSVGD} uses a number of non-random and interacting particles to represent the variational posterior distribution of the model's parameters. The particles are updated iteratively by each scheduled client using \ac{SVGD} \cite{10.5555/3157096.3157362}. We propose and analyze two client selection approaches for DSVGD that lead to faster convergence and lower communication cost.

\subsection{Contributions}
In this paper, we propose two selection schemes for \ac{DSVGD} based on the \ac{KSD} \cite{pmlr-v48-liub16} and the \ac{HIP} \cite{berlinet2011reproducing} for federated Bayesian learning. The main contributions of this paper are summarized as follows:

\begin{itemize}
\item We first propose a client selection scheme based on the \ac{KSD} between the global posterior and the local tilted distribution.  The client with a larger \ac{KSD}  will have a higher probability of being selected since it provides updates that maximize the decrease of the local free energy per iteration.
\item We also propose a scheme based on the \ac{HIP} between the SVGD update function for the global and local likelihood. We derive the closed form of \ac{HIP} using the properties of the \ac{RKHS}. The larger the \ac{HIP} the higher the probability of a client being selected per iteration. 
\item For both schemes, we derive the upper bound of the decrease of the global free energy between any two consecutive iterations. It is revealed that both schemes improve the model convergence, where the KSD-based scheme has smaller communication overhead, but the HIP-based scheme has faster convergence.
\end{itemize}

The rest of the paper is organized as follows. Section \ref{sec:system_model_and_learning_mechanism} introduces the federated Bayesian learning framework and DSVGD. Section \ref{sec:Scheduling_Schemes_for_Federated_Bayesian_Learning} proposes two client selection  schemes and the corresponding convergence analysis. Section \ref{sec:Experimental_Results} provides experimental results, followed by conclusions in Section \ref{sec:Conclusion}. 

\section {System Model and Preliminaries}
\label{sec:system_model_and_learning_mechanism}
In this section, we first describe the considered system model and then introduce the technical preliminaries used in our paper. The key notations that are important for the system model and the technical preliminaries are summarized in Table \ref{table1}.


\begin{table*}[t]
 \centering
 \caption{Key notations.}
\label{table1}
\begin{tabular}{|c|c|c|c|l}
\cline{1-4}
\textbf{Symbol}          & \textbf{Definitions}                                  & \textbf{Symbol} & \textbf{Definitions}                             &  \\ \cline{1-4}
$\theta$                    & model parameter vector                                & $F(q(\theta))$         & global free energy                               &  \\ \cline{1-4}
$K$ & number of participated clients & $\mathcal{A}_p \boldsymbol{f}(\theta)$          & Stein’s operator                                 &  \\ \cline{1-4}
$D_k$                      & dataset of client $k$                                   & $\mathbb{D}(p,q)$         & Stein discrepancy                                &  \\ \cline{1-4}
$L_k$                       & loss function of client $k$                             & $\mathbb{S}(p,q)$        &  Kernelized Stein Discrepancy                     &  \\ \cline{1-4}
$\alpha$ &
  temperature level &
  $\{\theta_n\}_{n=1}^N$ &
  \begin{tabular}[c]{@{}c@{}}global particles representing $q(\theta)$\end{tabular} &
   \\ \cline{1-4}
$p_0(\theta)$ &
  prior distribution &
  $\{\theta_{n,k}\}_{n=1}^N$ &
  \begin{tabular}[c]{@{}c@{}}local particles representing $t_k(\theta)$\end{tabular} &
   \\ \cline{1-4}
$p_k(\theta)$                   & local likelihood at client $k$                          & $\mathrm{k}(\theta, \theta')$        & kernel function used to define a RKHS $\mathcal{H}^d$           &  \\ \cline{1-4}
$q(\theta)$                     & variational global posterior distribution             & $\mathcal{H}^d$              & RKHS defined by a kernel $\mathrm{k}(\theta, \theta')$                                          &  \\ \cline{1-4}
$q_{opt}(\theta)$ &
  optimal global posterior &
  $\phi(\cdot)$ &
  \begin{tabular}[c]{@{}c@{}}SVGD update function restricted to a zero-centered ball of $\mathcal{H}^d$ \end{tabular} &
   \\ \cline{1-4}
$\tilde{q}_{opt}(\theta)$ &
  unnormalized optimal global posterior &
  $\phi(\cdot)^{*}$ &
  \begin{tabular}[c]{@{}c@{}}the target SVGD update function\end{tabular} &
   \\ \cline{1-4}
$\tilde{p}_k(\theta)$                   & local tilted distribution at client k                 & $\mathrm{K}(\theta, \theta_n)$      & kernel function used for kernel density estimate &  \\ \cline{1-4}
$t_k(\theta)$ &
  \begin{tabular}[c]{@{}c@{}}approximate scaled local likelihood at client $k$\end{tabular} &
  $I$ &
  global iteration numbers &
   \\ \cline{1-4}
$P_k(\theta)$ &
  \begin{tabular}[c]{@{}c@{}}selection probability of client $k$\end{tabular} &
  $L$ &
  local SVGD iterations numbers &
   \\ \cline{1-4}
$\epsilon$ &
  \begin{tabular}[c]{@{}c@{}}learning rate of SVGD\end{tabular} &
  $i$ &
  index of global iteration &
   \\ \cline{1-4}
$F_k(q(\theta))$                 & local free energy at client $k$                         & $l$               & index of local training iteration               &  \\ \cline{1-4}
\end{tabular}
\end{table*}

\subsection{Federated Bayesian Learning}
\begin{figure}[t]
\begin{centering}
\includegraphics[width=1.025\linewidth]{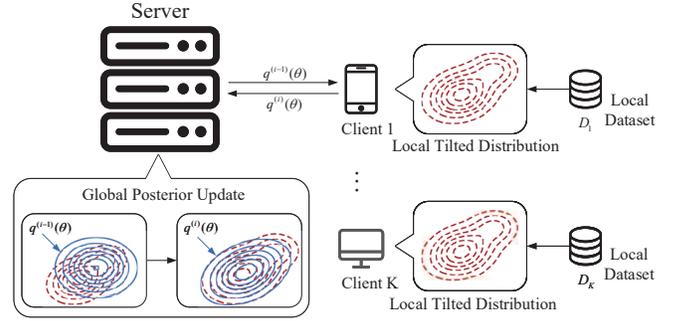}
\vspace{-0.2cm}
\caption{Federated Bayesian learning system.}\label{fig:System}
\end{centering}
\end{figure}
\label{sec:system_model_federated_bayesian_learning}
As shown in Fig. \ref{fig:System}, we consider a system setup for federated Bayesian learning\cite{vehtari2020expectation}, where a central server maintains the variational global posterior distribution $q(\theta)$ where $\theta$ is the model parameter vector, for e.g., weights of a neural network. Furthermore, $K$ clients, based on their local dataset $D_k$, exchange their own view of the variational posterior $q(\theta)$ with the server. Following the Bayesian learning framework \cite{knoblauch2019generalized}, the goal is to minimize the global free energy over the space of distributions as:
\begin{align}\label{eqn:glo_energy}
\min_{q(\theta)} \left\{F(q(\theta))=\sum_{k=1}^{K} \mathbb{E}_{\theta \sim q(\theta)}[L_k(\theta)] + \alpha \mathrm{KL}(q(\theta) \rVert p_0(\theta))\right\},
\end{align}
where $L_k(\theta)$ is the local training loss for model parameter $\theta$. The global free energy reflects a trade-off between the sum loss function over all clients data and the complexity defined by the Kullback-Leibler (KL) divergence between $q(\theta)$ and known prior distribution $p_0(\theta)$. The trade-off the two terms is dictated by the temperature level $\alpha$. Using simple algebraic manipulations, the global free energy in \eqref{eqn:glo_energy} can re-written as
\begin{align}\label{eqn:glo_kl}
\min_{q(\theta)} \left\{F(q(\theta))=\alpha\mathrm{KL}(q(\theta) \rVert \tilde{q}_{opt}(\theta))\right\},
\end{align}
where $\tilde{q}_{opt}(\theta)$ is the unnormalized optimal global posterior define as 
\begin{align}
    \tilde{q}_{opt} = p_0(\theta)\prod_{k=1}^K p_k(\theta) \label{eqn:q_tilde_opt},
\end{align}
where $p_0(\theta)$ is the prior distribution and $p_k(\theta)$ is the local likelihood at client $k$ defined as $p_k(\theta) =\exp(-\alpha^{-1}L_k(\theta))$. \eqref{eqn:glo_kl} shows visually the optimization target of the variational inference. Specifically,  Bayesian inference aims to obtain the posterior distribution, which is generally referred to as the target distribution, over the model parameters. Variational inference uses KL divergence to measure the similarity between two distributions and frames the Bayesian inference problem into a deterministic optimization that approximates the target distribution with a simpler distribution by minimizing their KL divergence. Therefore, global free energy refers to the KL divergence to be optimized in the variational inference, which is the KL divergence between the variational global posterior $q(\theta)$ and the optimal global posterior $\tilde{q}_{opt}(\theta)$, as shown in \eqref{eqn:glo_kl}.
By normalizing $\tilde{q}_{opt}(\theta)$ with the normalization constant $Z$ as $q_{opt}(\theta) = \frac{\tilde{q}_{opt}(\theta)}{Z}$ and setting $q(\theta) = q_{opt}(\theta) \propto \tilde{q}_{opt}(\theta)$, the global free energy can be minimized. However, this is not possible for two main reasons: first, computing $q_{opt}(\theta)$ requires the normalization constant $Z$ which involves computing a multidimensional integral and thus is not feasible in practice. Second, from \eqref{eqn:q_tilde_opt} and the definition of $p_k(\theta)$, computing $\tilde{q}_{opt}$ requires access to all datasets of the $K$ clients in order to compute $p_k (\theta)$.\par
Following the federated Bayesian learning framework in \cite{kassab2022federated}, we restrict $q(\theta)$ to the family of distributions having the following form:
\begin{align}\label{eqn:global_likelihood}
q(\theta)=p_0(\theta)\prod_{k=1}^K t_k(\theta).
\end{align}
The global free energy minimization is then conducted in a distributed manner minimizing the local free energy at each selected client as 
\begin{align}\label{eqn:local_kl}
\min_{q(\theta)} \left\{F_k(q(\theta))=\alpha\mathrm{KL}(q(\theta) \rVert \tilde{p}_k(\theta))\right\},
\end{align}
where 
\begin{align}\label{eqn:pk_update}
\tilde{p}_k(\theta)=\frac{q(\theta)}{t_k(\theta)}\cdot p_k(\theta) 
\end{align}
is the local tilted distribution and $t_k(\theta)$ approximates the scaled local likelihood $\frac{p_k}{Z_k}$ where $Z =\prod_{k=1}^KZ_k$. It is shown in \cite{bui2018partitioned} that minimizing the local free energies \eqref{eqn:local_kl} across all clients leads to the minimum of the global free energy in \eqref{eqn:glo_kl} by suitably combining the local energies minimums. Note that the clients contribute differently to the decrease of the global free energy, which can be accelerated by preferentially selecting clients with higher contribution. By definition of $q(\theta)$ in \eqref{eqn:global_likelihood}, the tilted distribution $\tilde{p}_k$ in \eqref{eqn:pk_update}, the tilted distribution $\tilde{p}_k$ contains the true likelihood $p_k(\theta)$ of client $k$ and the approximate scaled likelihoods $t_k'$ for $k' \neq k$, of the remaining clients. Once the local free energy minimization problem in \eqref{eqn:local_kl} is solved by the selected client $k$ at iteration $i$, a new global variational posterior $q^{(i)}(\theta)$ is obtained and the approximate scaled local likelihood is updated as 

\begin{align}\label{eqn:tk}
t_k^{(i)}=\frac{q^{(i)}(\theta)}{q^{(i-1)}(\theta)}t_k^{(i-1)}(\theta),
\end{align}
for each selected client $k$, while $t_{k^\prime}^{(i)}(\theta) = t_{k^\prime}^{(i-1)}(\theta)$ for clients $k^\prime$ that are not selected.

To summarize, the  general process of federated Bayesian learning is as follows: 
\begin{itemize}
\item{\bf Step 1:} The server serially selects a client and sends the current variational global posterior $q^{(i-1)}$ to the selected client.
\item{\bf Step 2:} The selected client minimizes the local free energy in \eqref{eqn:local_kl} to obtain the current variational global posterior $q^{(i)}(\theta)$.
\item{\bf Step 3:} The updated variational global posterior $q^{(i)} (\theta)$ is uploaded to the server and the approximate scaled local likelihood is updated using \eqref{eqn:tk} at the client side.
\item{\bf Step 4:} The server updates the current variational global posterior by setting $q^{(i-1)}(\theta) \longleftarrow q^{(i)}(\theta)$ and continues to serially select the next client.
\end{itemize}

\subsection{Preliminaries }
In this paper, we  investigate the client selection problem in DSVGD, which is a novel federated Bayesian learning framework and applies SVGD to the federated Bayesian learning framework presented in Section II-A to solve the local free energy in \eqref{eqn:local_kl}. Next, we illustrate the preliminaries of KSD, SVGD, and DSVGD  that are used in our algorithms.

\subsubsection{Kernel Stein Discrepancy} 
We start by introducing two important definitions in the Stein's method, namely the Stein's operator and the Stein's identity. Stein's operator is defined as $\mathcal{A}_p \boldsymbol{f}(\theta)=[\nabla\log p(\theta) \boldsymbol{f}(\theta)^\top + \nabla \boldsymbol{f}(\theta)]$, where $p(\theta)$ is a  smooth density and $\boldsymbol{f}=[f_1(\theta),...,f_d(\theta)]^\top$is a smooth vector function. While Stein's identity can be written as $\mathbb{E}_{\theta \sim p}[\mathcal{A}_p \boldsymbol{f}(\theta)]=0$, it is generally used to measure the divergence between two distributions. Given two distributions $p(\theta)$ and $q(\theta)$ where $q(\theta)$ is possibly unnormalized, the divergence is defined as the maximum of the expectation for trace of the Stein's operator $\mathcal{A}_q \boldsymbol{f}(\theta)$ under $\theta\sim p(\theta)$ as 
\begin{align}\label{eqn:Stein_d}
\mathbb{D}(p,q)=\max_{\boldsymbol{f}\in \mathcal{H}^d} \{\mathbb{E}_{\theta \sim p}[\mathrm{trace}(\mathcal{A}_q \boldsymbol{f}(\theta))]\},
\end{align}
where $\boldsymbol{f}$  is restricted to lie within the unit ball of a RKHS $\mathcal{H}^d$ defined by a positive definite kernel $\mathrm{k}(\theta,\theta')$. According to \cite{10.5555/3157096.3157362}, we can rewrite \eqref{eqn:Stein_d} using the properties of RKHS as
\begin{align}\label{eqn:ksd_rkhs}
\mathbb{D}(p,q)=\max_{\boldsymbol{f}\in \mathcal{H}^d} \{ \langle \boldsymbol{f},\phi^{*} \rangle_{\mathcal{H}^d},~~s.t.~~\|\phi^{*}\|_{\mathcal{H}^d}\leq 1   \},
\end{align}
where $\phi^{*}(\cdot)=\mathbb{E}_{\theta \sim p}[\nabla\log q(\theta) \mathrm{k}(\theta,\cdot)+\nabla \mathrm{k}(\theta,\cdot)]$. According to the properties of the inner product, the maximum value of $\langle \boldsymbol{f},\phi^{*}\rangle_{\mathcal{H}^d}$ is obtained as $\|\phi^{*}\|_{\mathcal{H}^d}$ when $\boldsymbol{f}$ is in the same direction as $\phi^{*}$. In \cite{pmlr-v48-liub16}, the KSD is defined as the square of this maximum value:
\begin{align}\label{eqn:KSD}
\mathbb{S}(p,q)=\|\phi^{*}\|_{\mathcal{H}^d}^2.
\end{align}

\subsubsection{Stein Variational Gradient Descent} \ac{SVGD} \cite{10.5555/3157096.3157362} is a non-parametric variational inference algorithm that aims to solve optimization problems of a similar form to \eqref{eqn:local_kl}. \ac{SVGD} encodes  the variational distribution $q(\theta)$ using a set of particles $\{\theta_n\}_{n=1}^N$ that are updated iteratively to match the (unnormalized) target distribution $\tilde{p}(\theta)$ as
\begin{align}\label{eqn:svgd}
\theta_n^{[l]}\leftarrow & \theta_n^{[l-1]}+\epsilon \phi(\theta_n^{[l-1]}),
\end{align}
where the SVGD update function $\phi(\cdot)$ is optimized to maximize the steepest descent decrease of the KL divergence between $q(\theta) $ and the target distribution $\tilde{p}(\theta)$ as follows
\begin{align}\label{eqn:SVGD_KL}
\phi(\cdot) = \mathop{\arg\max}_{\phi(\cdot) \in \mathcal{H}^d}\bigg\{-\frac{d}{d\epsilon}&\mathrm{KL}(q(\theta)\|\tilde{p}(\theta)) \nonumber\\
&=\mathbb{E}_{\theta \sim q}[\mathrm{trace}(\mathcal{A}_{\tilde{p}} \phi(\theta))]\bigg\}.
\end{align}
As shown in \eqref{eqn:SVGD_KL}, the negative gradient of the KL divergence has the same form as KSD, so we can rewrite  the term $\mathbb{E}_{\theta \sim q}[\mathrm{trace}(\mathcal{A}_{\tilde{p}} \phi(\theta))]$ in the form of HIP as $\langle \phi,\phi^*\rangle_{\mathcal{H}^d}$, where function $\phi^*(\cdot)=\mathbb{E}_{\theta \sim q}[\nabla\log \tilde{p}(\theta) \mathrm{k}(\theta,\cdot)+\nabla \mathrm{k}(\theta,\cdot)]$ and the update function $\phi(\cdot)$ is restricted to zero-centered balls of a RKHS $\mathcal{H}^d$ defined by a kernel $\mathrm{k}(\theta, \theta')$, where the norm of the ball is taken to be $\|\phi^*\|_{\mathcal{H}^d}^2$. Accordingly, the HIP can achieve the maximum value when the update function $\phi(\theta)$ takes the same norm and direction as the function $\phi^*(\theta)$ so that the gradient of the KL divergence can achieve the maximum value, which allows the particles to update towards the target distribution at the fastest speed. Therefore, SVGD takes  $\phi(\cdot)=\mathbb{E}_{\theta \sim q}[\nabla\log \tilde{p}(\theta) \mathrm{k}(\theta,\cdot)+\nabla \mathrm{k}(\theta,\cdot)]$ as the update function for the normalized distribution $p(\theta) \propto \tilde{p}(\theta)$ to update the set of particles $\{\theta_n\}_{n=1}^N$. The vartional distribution is obtained using a \ac{KDE} over the particles as $q(\theta) = \sum_{n=1}^{N} K(\theta, \theta_n)$ and by controlling the number of particles, SVGD can trade off the bias of the vartional distribution and the speed of convergence, providing flexible performance.

\subsubsection{Distributed Stein Variational Gradient Descent} \ac{DSVGD} inherits the performance of SVGD in terms of bias, convergence speed, and per-iteration complexity by controlling the number of particles \cite{kassab2022federated}. At each global iteration $i=1,2,\ldots, I$, DSVGD maintains a set of global particles $\{\theta_n^{(i-1)}\}_{n=1}^N$ at the server side, which represent the approximate posterior $q^{(i-1)}(\theta)$ of the current model parameters, and a set of local particles $\{\theta_{n,k}^{(i-1)}\}_{n=1}^N$ at the client side, which represent the approximate scaled local likelihood $t_k^{(i-1)}(\theta)$. Both the global and local particles are updated using \ac{SVGD} as follows: first, the global particles $\{\theta_n^{(i-1)}\}_{n=1}^N$ are downloaded by the client $k$ and updated by running $L$ local SVGD iterations to approximate the local tilted distribution $\tilde{p}_k^{(i)}(\theta)$ as
\begin{align}\label{eqn:glocal_svgd}
\theta_n^{[l]}\leftarrow\theta_n^{[l-1]} &+ \frac{\epsilon}{N}\sum_{j=1}^{N}[\mathrm{k}(\theta_j^{[l-1]},\theta_n^{[l-1]})\nabla_{\theta_j}\log \tilde{p}_k^{(i)}(\theta_j^{[l-1]}) 
\nonumber \\
&+ \nabla_{\theta_j}\mathrm{k}(\theta_j^{[l-1]},\theta_n^{[l-1]})],
\end{align}
for $l=1,2,\ldots, L$ and where the global approximate posterior is then obtained using a \ac{KDE} of the global particles as 
$q(\theta)=N^{-1}\sum_{n=1}^N\mathrm{K}(\theta,\theta_{n})$.  Similarly, the local particles are updated by running $L'$ local SVGD iterations to approximate the distribution $t_k^{(i)}(\theta)$ as
\begin{align}\label{eqn:local_svgd}
\theta_{k,n}^{[l']}\leftarrow\theta_{k,n}^{[l'-1]} &+ \frac{\epsilon'}{N}\sum_{j=1}^{N}[\mathrm{k}(\theta_{k,j}^{[l'-1]},\theta_{k,n}^{[l'-1]})\nabla_{\theta_j}\log t_k^{(i)}(\theta)\nonumber \\
&+ \nabla_{\theta_j}\mathrm{k}(\theta_{k,j}^{[l'-1]},\theta_{k,n}^{[l'-1]})],
\end{align}
where $t_k^{(i)}(\theta)$ is computed using \eqref{eqn:tk}. To summarize, the process of DSVGD is as follows:
\begin{itemize}
\item{\bf Step 1:} The server serially selects a client and sends the current global particles $\{\theta_n^{(i-1)}\}_{n=1}^N$  to the selected client.
\item{\bf Step 2:} The selected client minimizes the local free energy in \eqref{eqn:local_kl} by running SVGD algorithm (see \eqref{eqn:glocal_svgd}) to obtain the current global particles $\{\theta_n^{(i)}\}_{n=1}^N$.
\item{\bf Step 3:} The updated global particles $\{\theta_n^{(i)}\}_{n=1}^N$ are uploaded to the server and the local particles $\{\theta_{n,k}^{(i-1)}\}_{n=1}^N$ are updated by running SVGD algorithm (see \eqref{eqn:local_svgd}) to obtain the current local particles $\{\theta_{n,k}^{(i)}\}_{n=1}^N$  at the client side.
\item{\bf Step 4:} The server updates the current global particles by setting $\{\theta_n^{(i-1)}\}_{n=1}^N \longleftarrow \{\theta_n^{(i)}\}_{n=1}^N$ and continues to serially select the next client.
\end{itemize}

In this paper, we consider the selection problem for DSVGD which aims to select the client that contributes most to the decrease of the global free energy in each global iteration and thus improving convergence.
\section{Client Selection Schemes for Federated Bayesian Learning}
\label{sec:Scheduling_Schemes_for_Federated_Bayesian_Learning}
In this section, we start by introducing a generic client selection framework for \ac{DSVGD}. Then, we propose two schemes, namely KSD-based scheme and HIP-based scheme. Moreover, we analyze the convergence of \ac{DSVGD} which reveals that both schemes improve the model convergence, where the KSD-based scheme has smaller communication overhead and the HIP-based scheme has faster convergence.
\subsection{Generic Client Selection Framework}

We now introduce the generic client selection framework for federated Bayesian learning where the server estimates the selection distribution $P_k$, for $k\in \{1,2,...,K\}$ which represents the probability of selecting client $k$. At each global iteration $i$,  a client $k\in \{1,2,...,K\}$ is selected based on the distribution $P_k$ to update the global posterior. The step of this framework can be summarized as follows:
\begin{itemize}
\item{\bf Step 1} (Global Posterior Broadcast): The server broadcasts the current global posterior to all clients.
\item{\bf Step 2} (Importance Indicator Report): Each client $k$ computes an indicator and reports the result to the server. This step is going to be detailed later.
\item{\bf Step 3} (Selection Distribution Calculation): Based on the received reports, the server computes the selection distribution $P_k$ of the clients.
\item{\bf Step 4} (Client  Selection): Based on the selection distribution, the server selects a client, and then solicits updates from the selected client.
\item{\bf Step 5} (Global Posterior Update): The  selected client minimizes local free energy in \eqref{eqn:local_kl} to obtain the new global posterior.
\item{\bf Step 6} (Global Posterior Upload and Approximate Scaled Local Likelihood Update): The selected client uploads the global posterior obtained in the previous step to the server and updates the  approximate scaled local likelihood using \eqref{eqn:tk}.
\end{itemize}

Note that the client selection requires the addition of three steps: the importance indicator report (Step 2), selection distribution  calculation (Step 3) and client selection (Step 4) to the traditional federated learning process presented in Sec. II-A. Also, the global posterior needs to be broadcast to all clients at the beginning of each communication iteration, instead of being sent serially to each client.

\subsection{KSD-Based Scheme}
Federated Bayesian learning converges as the variational global posterior $q(\theta)$ gradually approximates each local tilted distribution (see \eqref{eqn:local_kl}), where the discrepancy between each local tilted distribution and the current variational global posterior is different. Here, we measure this discrepancy in terms of KSD, which can be easily estimated from the local tilted distribution and the global particles representing the variational global posterior. A larger KSD between the two distributions indicates that the variational global posterior has not sufficiently approximated the local tilted distribution, and therefore the client may have contributed more to the decrease of the global free energy. A naive idea is to use KSD as an indicator for judging the importance of the client's data, assigning the client with larger KSD value a higher probability of selection. We next analyse the convergence of the global free energy by stating the following lemma.

\begin{lemma}\label{lemma1}
If a client $k$ is selected, from local iteration $l$ to $l+1$ during global iteration $i$, the upper bound of the decrease of the global free energy can be expressed as
\begin{align}\label{eqn:lemma1}
\left(F(q^{[l+1]}(\theta))-F(q^{[l]}(\theta))\right)/\alpha \leq -\epsilon S(q^{[l]},\tilde{p}_k) ~~~~~&\nonumber\\
+2(K-1)l_{max}^{(i)} \sqrt{2\mathrm{KL}(q^{[l+1]}\rVert q^{[l]})}&,
\end{align}
where $l_{max}^{(i)}=\sup_{\theta}\max_{m \ne k}\left| \log (t_m^{(i-1)}(\theta)) \cdot p_m(\theta) \right|$. As a corollary, after $I$ rounds,
\begin{align}
&\left(F(q^{(I)}(\theta)) -F(q^{(0)}(\theta))\right)/\alpha \nonumber\\
&\leq -\epsilon\sum_{i=0}^{I}\sum_{k=1}^KP_k^{(i)}\sum_{l=0}^{L-1} S(q^{[l]},\tilde{p}_k) \nonumber\\
&\quad+\sum_{i=0}^{I}\sum_{k=1}^KP_k^{(i)}\sum_{l=0}^{L-1}2(K-1)l_{max}^{(i)} \sqrt{2\mathrm{KL}(q^{[l+1]}\rVert q^{[l]})},
\end{align}
where $P_k^{(i)}$ is the probability of selecting client $k$ in the $i$-th iteration. 
\end{lemma}
\begin{proof}
See Appendix A.
\end{proof}
The first term on the right hand side of \eqref{eqn:lemma1} corresponds to the upper bound of the decrease of the local free energy, expressed by the KSD between the local tilted distribution and the variational global posterior, while the second term of \eqref{eqn:lemma1} corresponds to the upper bound of the impact of un-selected clients on the global free energy. Therefore, Lemma \ref{lemma1} shows that the decrease of the global free energy is influenced by both the selected clients and un-selected clients, where the selected client influences the decrease of global free energy through the decrease of the local free energy. It can be observed that the larger the client's KSD is, the smaller the upper bound of the local free energy and the greater the contribution to the decrease of the local free energy. The idea is that the client with larger KSD will approximately reduce the upper bound of the decrease of the global free energy while ignoring the effect of un-selected clients, and thus we use KSD to discriminate the importance of the client's data, assigning a higher selection probability to the clients with larger KSD.

\begin{definition}\label{definition1}
The KSD-based selection distribution is defined as: 
\begin{align}\label{eqn:definition1}
P_{{\rm KSD}, k}=\frac{S_k(q,\tilde{p}_k)}{\sum_{m=1}^{K} S_m(q,\tilde{p}_m)}.
\end{align}
\end{definition}
Here we use the KSD between the variational global posterior $q(\theta)$ and the local tilted distribution $\tilde{p}_k(\theta)$ as an indicator for judging the importance of the clients' data. Intuitively, the client with a larger KSD can be prioritized based on the selection distribution in \eqref{eqn:definition1}. More specifically, the server can prioritise the selection of clients whose local tilted distribution is not sufficiently approximated by the variational global posterior.
We present the KSD-based scheme in Algorithm \ref{alg:1} where the step 4 to the step 6 correspond to the extra steps added by the selection framework with respect to the DSVGD algorithm.
\par
\textbf{Comparison to the round robin and random selection schemes:} With the impact of un-selected clients approximately ignored, the KSD-based scheme results in a tighter upper bound of the decrease of the global free energy compared to the round robin and random selection schemes, where the probability of each client being selected is set uniformly to $\frac{1}{K}$. Specifically, defining $\Delta_{\mathrm{KSD}}$ as the upper bound reduction from the KSD-based scheme compared to the round robin and random selection schemes, the bound of the KSD-based scheme is tighter if
\begin{align}
\Delta_{\mathrm{KSD}} = \epsilon \sum_{k=1}^{K}S_k(q,\tilde{p}_k)P_{\mathrm{KSD},k} - \frac{\epsilon}{K}\sum_{k=1}^{K}S_k(q,\tilde{p}_k) \geq 0,
\end{align}
which holds since
\begin{align*}
\epsilon \sum_{k=1}^{K}S_k(q,\tilde{p}_k)&P_{\mathrm{KSD},k} \nonumber\\
= &\epsilon\sum_{k=1}^K \frac{S_k(q,\tilde{p}_k)^2}{\sum_{m=1}^K S_m(q,\tilde{p}_m)} \nonumber \\
\geq &\frac{\epsilon}{K}\sum_{k=1}^{K}S_k(q,\tilde{p}_k),  \quad\mathrm{(Jensen ~Inequality)}
\end{align*}
where the equality holds under the condition $S_1(q,\tilde{p}_k) = S_2(q,\tilde{p}_k) = ... = S_K(q,\tilde{p}_k)$. In such a condition, the difference between each local tilted distribution and the current variational global posterior is the same, indicating that the clients' data is IID. In this case, $\Delta_{\mathrm{KSD}}=0$ and the KSD-based scheme no longer improves the performance of the model.
\begin{algorithm}
	\renewcommand{\algorithmicrequire}{\textbf{Input:}}
	\renewcommand{\algorithmicensure}{\textbf{Output:}}
	\caption{KSD-based client selection for DSVGD}
	\label{alg:1}
    \begin{algorithmic}[1]
        \REQUIRE Prior $p_0(\theta)$, local likelihood $\{p_k\}_{k=1}^K$, temperature $\alpha>0$, kernels $\mathrm{K}(\cdot,\cdot)$ and $\mathrm{k}(\cdot,\cdot)$.
        \ENSURE Global approximate posterior $q(\theta)=N^{-1}\sum_{n=1}^N\mathrm{K}(\theta,\theta_n)$.
        \STATE $\textbf{initialize}$ $q^{(0)}(\theta)=p_0(\theta)$; $\{\theta_n^{(0)}\}_{n=1}^N \overset{\mathrm{i.i.d}}{\sim}p_0(\theta)$; $\{\theta_{k,n}^{(0)}=\theta_n^{(0)}\}_{n=1}^N$ and $t_k^{(0)}(\theta)=1$ for $k=1,...,K$.
        \FOR {$i=1,\cdots$, $I$}
        \STATE Server sends $\{\theta_n^{(i-1)}\}_{n=1}^N$ to all clients;
        \STATE Each client $k$ computes its own KSD $S(q,\tilde{p}_k)$;
        \STATE Each client $k$ sends $S(q,\tilde{p}_k)$ back to the server;
        \STATE Server computes the selection distribution $P_{{\rm KSD}, k}^{(i)}$ using \eqref{eqn:definition1} and $S(q,\tilde{p}_k),k=1,...,K$;
        \STATE Server randomly selects a client $k$ according to $P_{{\rm KSD}, k}^{(i)},k=1,...,K$;
        \STATE Client $k$ obtains the updated global particles $\{\theta_n^{(i)}\}_{n=1}^N$ according to \eqref{eqn:glocal_svgd};
        \STATE Client $k$ sends the updated global particles $\{\theta_n^{(i)}\}_{n=1}^N$ to the server and obtains the updated local particles $\{\theta_{k,n}^{(i)}\}_{n=1}^N$ according to \eqref{eqn:local_svgd}.
        \ENDFOR
        \STATE \textbf{return} $q(\theta)=N^{-1}\sum_{n=1}^N\mathrm{K}(\theta,\theta_n^{(I)})$.
    \end{algorithmic}
\end{algorithm}

\subsection{HIP-based Scheme}
The KSD-based scheme proposed in the previous subsection exploits the divergence between distributions as an indicator of the importance of the clients' data. However, this scheme actually only achieves the minimization of the decrease of the local free energy per iteration, as it ignores the impact of un-selected clients on the global free energy for simplicity. In this subsection, we integrate the impact of  both the selected and un-selected clients (which corresponds to the first and second term on the right hand side in \eqref{eqn:lemma1}) on the global free energy through the HIP between two SVGD update functions.

We start by recalling the update function for particles in SVGD.  According to the previous discussion in Section II-B, the update function is restricted to a zero-centered ball of a RKHS $\mathcal{H}^d$ and is optimized to maximize the negative gradient of the KL divergence, which can be written in the form of the KSD as shown in \eqref{eqn:SVGD_KL},  and in turn can be further rewritten in the form of the HIP as  $\langle \phi,\phi^*\rangle_{\mathcal{H}^d}$  as shown in \eqref{eqn:ksd_rkhs}, where $\phi^*(\cdot)=\mathbb{E}_{\theta \sim q}[\nabla\log \tilde{p}_k(\theta) \mathrm{k}(\theta,\cdot)+\nabla \mathrm{k}(\theta,\cdot)]$ is referred to as the target update function for normalized distribution $p_k(\theta) \propto \tilde{p}_k(\theta)$. When the update function $\phi(\theta)$ is equal in size and direction to the target update function $\phi^*(\theta)$, the HIP can achieve a maximum value, when the norm of the ball is taken to be $\|\phi^*\|_{\mathcal{H}^d}^2$. Similarly, when we choose a different update function, the difference in both the size and direction between it and the target update function in Hilbert space will cause the HIP to change, which in turn affects the speed at which the particles are updated towards the target distribution. Therefore, we can use the HIP between the update function $\phi(\theta)$ and the target update function $\phi(\theta^*)$ to measure the contribution of the particles update. A larger HIP indicates that the update function can make the variational distribution approach the target distribution at a faster speed, which makes a greater contribution to the particles update. This is formalized in the following theorem.

\begin{theorem}\label{theorem1}
Given i.i.d. samples $\left\{\theta_n \right\}_{n=1}^N$ drawn from the distribution  $q(\theta)$ and the score function $S_{p1}$ and $S_{p2}$. We can estimate the HIP between the SVGD update function  $\phi_1(\cdot)=\mathbb{E}_{\theta \sim q}[\nabla\log p_1(\theta) \mathrm{k}(\theta,\cdot)+\nabla \mathrm{k}(\theta,\cdot)]$ for $p_1$ and the SVGD update function  $\phi_2(\cdot)=\mathbb{E}_{\theta \sim q}[\nabla\log p_2(\theta) \mathrm{k}(\theta,\cdot)+\nabla \mathrm{k}(\theta,\cdot)]$ for $p_2$ by using the properties of the RKHS. Define
\begin{align*}
h_{p_1,p_2}(\theta_i,\theta_j)&= S_{p_1}(\theta_i)^T \mathrm{k}(\theta_i,\theta_j)S_{p_2}(\theta_j) \nonumber\\
+& S_{p_1}(\theta_i)^T \nabla_{\theta_j}\mathrm{k}(\theta_i,\theta_j)  
+ \nabla_{\theta_i}\mathrm{k}(\theta_i,\theta_j)^T S_{p_2}(\theta_j) \nonumber\\
+& {\rm trace}(\nabla_{\theta_i,\theta_j}\mathrm{k}(\theta_i,\theta_j)),
 \end{align*}
then the HIP can be estimated as
\begin{align}\label{eqn:theorem1}
\langle \phi_1,\phi_2 \rangle_{\mathcal{H}^d}=\frac{1}{N^2}\sum_{i,j=1}^N[h_{p_1,p_2}(\theta_i,\theta_j)].
\end{align}
\end{theorem}
\begin{proof}
See Appendix B.
\end{proof}
Theorem \ref{theorem1} reveals that we can estimate the HIP between the SVGD update function for $p_1$ and the target update function for target  distribution $p_2$ using \eqref{eqn:theorem1}, and thus use the HIP estimate to measure the contribution of the SVGD update function for $p_1$ to the update of particles towards the target distribution $p_2$.

\begin{lemma}\label{lemma2}
Given $\left \langle \phi_k,\phi_m \right \rangle_{\mathcal{H}^d}$ representing the HIP between the update functions for the likelihoods of client $k$ and client $m$, the upper bound of the decrease of the global free energy per iteration can be written as
\begin{align}\label{eqn:lemma2}
\left(F(q^{[l+1]}(\theta))-F(q^{[l]}(\theta))\right) /&\alpha\leq \nonumber\\
-\epsilon \sum_{m=1}^K \left \langle \phi_k,\phi_m \right \rangle_{\mathcal{H}^d} + \sum_{m=1}^K &\mathrm{KL}(q^{[l]}\rVert t_m^{(i-1)})+\log C,
\end{align}
where $C=\prod_{m=1}^K C_m$ is a normalization constant such that $\int\frac{t_m(\theta)}{C_k} \mathrm{d}\theta = 1$ for $m=1,\ldots, K$.
\end{lemma}
\begin{proof}
See Appendix C.
\end{proof}
Here in Lemma \ref{lemma2}, and in contrast with Lemma \ref{lemma1}, only the first term in the upper bound is related to the selected client, and thus,  the second term and the third term can be ignored. The first term is written in the from of the HIP, which is the sum of HIP between the update functions of the selected client $k$ and all clients including both the selected and un-selected clients. Specifically, the HIP term can be divided into $\left \langle \phi_k,\phi_k \right \rangle_{\mathcal{H}^d}$ and $\sum_{m\neq k}\left \langle \phi_k,\phi_m \right \rangle_{\mathcal{H}^d}$. The term $\left \langle \phi_k,\phi_k \right \rangle_{\mathcal{H}^d}$ can be rewritten in the form of KSD as $S(q,p_k)$, which represents the discrepancy between the local likelihood $p_k$ of the selected client $k$ and the variational global posterior $q$ represented by the global particles. In addition, the impact of an un-selected client $m$ is expressed as $\left \langle \phi_k,\phi_m \right \rangle_{\mathcal{H}^d}$, which is the contribution of the SVGD update functions of the selected client $k$ to the update of particles towards the local likelihood $p_m$ of un-selected client $m$. Therefore, the term $\sum_{m \neq k} \left \langle \phi_k,\phi_m \right \rangle_{\mathcal{H}^d}$ manifest the impact of all un-selected clients by aggregating the contributions of the selected client $k$ that causes particles to update in different directions. Lemma \ref{lemma2} reveals that the larger the HIP term the smaller the upper bound on the decrease of the global free energy per iteration, indicating that the selected client can make a greater contribution to the decrease of the global free energy. We now state a lemma that will prove useful to further simplify the upper bound in lemma \ref{lemma2}.

\begin{lemma}\label{lemma3}
Given the distribution $p'=\left(\prod_{m=1}^K p_m \right)^{\frac{1}{K}}$ and the SVGD update function for $p'$ as $\phi'(\cdot)=\mathbb{E}_{\theta \sim q}[\nabla\log p'(\theta) \mathrm{k}(\theta,\cdot)+\nabla \mathrm{k}(\theta,\cdot)]$, we can further rewrite the HIP term $\sum_{m=1}^K \left \langle \phi_k,\phi_m \right \rangle_{\mathcal{H}^d}$ as
\begin{align}\label{eqn:lemma3}
\sum_{m=1}^K \left \langle \phi_k,\phi_m \right \rangle_{\mathcal{H}^d} = K \cdot \left \langle \phi_k,\phi' \right \rangle_{\mathcal{H}^d}.
\end{align}
\end{lemma}
\begin{proof}
See Appendix D.
\end{proof}
Lemma \ref{lemma3} transforms the sum of HIP into the HIP between the update function for the likelihood of client $k$ and the update function of the global averaged likelihood  $p'=\left(\prod_{m=1}^K p_m \right)^{\frac{1}{K}}$. According to the previous discussion, since $p'$ is a geometric average of the local likelihoods across $K$ clients,  $p'$ contains the likelihood information of all clients. Therefore, we define the global averaged likelihood as the target distribution, the sum of HIP in \eqref{eqn:lemma2} represents the contribution of client $k$ to the update of the particles towards the global averaged likelihood, which contains information about all clients. 

\begin{theorem}\label{theorem2}
Using Lemma \ref{lemma3}, we can further simplify the upper bound of the decrease of the global free energy in \eqref{eqn:lemma2} as
\begin{align}\label{eqn:theorem2}
\left(F(q^{[l+1]}(\theta))-F(q^{[l]}(\theta))\right)/\alpha\leq &-\epsilon K \cdot \left \langle \phi_k,\phi' \right \rangle_{\mathcal{H}^d} \nonumber\\
+ \sum_{m=1}^K \mathrm{KL} &(q^{[l]}\rVert t_m^{(i-1)}) + \log C.
\end{align}
As a corollary, after $I$ rounds,
\begin{align}
&\left(F(q^{(I)}(\theta)) -F(q^{(0)}(\theta))\right)/\alpha \nonumber\\
&\leq -\epsilon \sum_{i=0}^{I}\sum_{k=1}^KP_k^{(i)}\sum_{l=0}^{L-1} K \cdot \left \langle \phi_k,\phi' \right \rangle_{\mathcal{H}^d}  \nonumber\\
&\quad+ \sum_{i=0}^{I}\sum_{l=0}^{L-1}\left(\sum_{m=1}^K \mathrm{KL}(q^{[l]}\rVert t_m^{(i-1)}) + \log C\right),
\end{align}
where $P_k^{(i)}$ is the probability of selecting client $k$ in the $i$-th iteration. 
\end{theorem}
Here only the HIP term $\left \langle \phi_k,\phi' \right \rangle_{\mathcal{H}^d}$ is relevant to the client selection. Theorem \ref{theorem2} reveals that the upper bound is smaller when the HIP term $\left \langle \phi_k,\phi' \right \rangle_{\mathcal{H}^d}$ is larger, indicating that the client has a greater contribution to the global particles update if the client can get the particles to update towards the global averaged likelihood $p'$, which contains information about all clients, at a faster speed. Therefore, we consider the HIP term $\left \langle \phi_k,\phi' \right \rangle_{\mathcal{H}^d}$  as a measure of the importance of the clients' data and assign higher selection probabilities to the clients with larger HIP.
\begin{definition}\label{definition2}
The HIP-based selection distribution is defined as: 
\begin{align}\label{eqn:definition2}
P_{{\rm HIP},k}=\frac{\left \langle \phi_k,\phi' \right \rangle_{\mathcal{H}^d}}{\sum_{m=1}^K \left \langle \phi_m,\phi' \right \rangle_{\mathcal{H}^d}}.
\end{align}
\end{definition}
Here we use the HIP between the update function for local likelihood and the update function for the global averaged likelihood $p'=\left(\prod_{m=1}^K p_m \right)^{\frac{1}{K}}$ as an indicator for judging the importance of a client's  data. Intuitively, using \eqref{eqn:definition2}, the server can prioritize the client who updates particles towards the global averaged likelihood at a faster speed, maximizing the decrease of the global free energy per iteration. Finally, we present the HIP-based scheme in Algorithm \ref{alg:2} where the step 4 to the step 6 corresponds the extra steps with respect to DSVGD and are added for selecting.
\par
\textbf{Comparison to the round robin and random selection schemes:} The HIP-based scheme results in a tighter upper bound of the decrease of the global free energy compared to the round robin and random selection schemes. Specifically, we define $\Delta_{\mathrm{HIP}}$ as the upper bound reduction from the HIP-based scheme compared to the round robin and random selection schemes, the bound of the HIP-based scheme is tighter if
\begin{align}
\Delta_{\mathrm{HIP}} = \epsilon \sum_{k=1}^{K}\left \langle \phi_k,\phi' \right \rangle_{\mathcal{H}^d}P_{\mathrm{HIP},k} - \frac{\epsilon}{K}\sum_{k=1}^{K}\left \langle \phi_k,\phi' \right \rangle_{\mathcal{H}^d} \geq 0,
\end{align}
which holds since
\begin{align*}
\epsilon \sum_{k=1}^{K}\left \langle \phi_k,\phi' \right \rangle_{\mathcal{H}^d}&P_{\mathrm{HIP},k} \nonumber\\
= &\epsilon\sum_{k=1}^K \frac{\left \langle \phi_k,\phi' \right \rangle_{\mathcal{H}^d}^2}{\sum_{m=1}^K \left \langle \phi_m,\phi' \right \rangle_{\mathcal{H}^d}} \nonumber \\
\geq &\frac{\epsilon}{K}\sum_{k=1}^{K}\left \langle \phi_k,\phi' \right \rangle_{\mathcal{H}^d}, \quad \mathrm{(Jensen ~Inequality)}
\end{align*}
where the equality holds under the condition $\left \langle \phi_1,\phi' \right \rangle_{\mathcal{H}^d} = \left \langle \phi_2,\phi' \right \rangle_{\mathcal{H}^d} = ... = \left \langle \phi_K,\phi' \right \rangle_{\mathcal{H}^d}$. In such a condition, the HIP between the local update function of each client and the global update function is the same, indicating that the clients' data is IID. In this case, $\Delta_{\mathrm{HIP}}=0$ and the HIP-based scheme no longer improves the performance of the model.

\begin{algorithm}
	\renewcommand{\algorithmicrequire}{\textbf{Input:}}
	\renewcommand{\algorithmicensure}{\textbf{Output:}}
	\caption{HIP-based client selection scheme for DSVGD}
	\label{alg:2}
    \begin{algorithmic}[1]
        \REQUIRE Prior $p_0(\theta)$, local likelihood $\{p_k\}_{k=1}^K$, temperature $\alpha>0$, kernels $\mathrm{K}(\cdot,\cdot)$ and $\mathrm{k}(\cdot,\cdot)$.
        \ENSURE Global approximate posterior $q(\theta)=N^{-1}\sum_{n=1}^N\mathrm{K}(\theta,\theta_n)$.
        \STATE $\textbf{initialize}$ $q^{(0)}(\theta)=p_0(\theta)$; $\{\theta_n^{(0)}\}_{n=1}^N \overset{\mathrm{i.i.d}}{\sim}p_0(\theta)$; $\{\theta_{k,n}^{(0)}=\theta_n^{(0)}\}_{n=1}^N$ and $t_k^{(0)}(\theta)=1$ for $k=1,...,K$.
        \FOR {$i=1,\cdots$, $I$}
        \STATE Server sends $\{\theta_n^{(i-1)}\}_{n=1}^N$ to all clients;
        \STATE Each client  $k$ computes its own score function gradient $\{\nabla_{\theta_n}\log p_k\}_{n=1}^N$;
        \STATE Each client $k$ sends $\{\nabla_{\theta_n}\log p_k\}_{n=1}^N$ back to the server;
        \STATE Server computes the client selection probability $P_{{\rm HIP},k}^{(i)}$ using \eqref{eqn:theorem1}, \eqref{eqn:definition2} and $\{\nabla_{\theta_n}\log p_k\}_{n=1}^N,k=1,...,K$;
        \STATE Server randomly selects a client $k$ according to $P_{{\rm HIP},k}^{(i)},k=1,...,K$;
        \STATE Client $k$ obtains updated global particles $\{\theta_n^{(i)}\}_{n=1}^N$ according to \eqref{eqn:glocal_svgd};
        \STATE Client $k$ sends the updated global particles $\{\theta_n^{(i)}\}_{n=1}^N$ to the server and obtains the updated local particles $\{\theta_{k,n}^{(i)}\}_{n=1}^N$ according to \eqref{eqn:local_svgd}.
        \ENDFOR
        \STATE \textbf{return} $q(\theta)=N^{-1}\sum_{n=1}^N\mathrm{K}(\theta,\theta_n^{(I)})$.
    \end{algorithmic}
\end{algorithm}

\section{Experimental Results}
\label{sec:Experimental_Results}
In this section, we conduct experiments to validate the theoretical analysis and test the performance of the proposed schemes.
\subsection{Experiment Settings}
As in \cite{kassab2022federated} and \cite{10.5555/3157096.3157362}, for all our experiments with SVGD and DSVGD, we use the Radial Basis Function (RBF) kernel $\mathrm{k}(x,x_0)=exp(-\|x-x_0\|_2^2/h)$ and take the bandwidth to be $h=\mathrm{med}^2/\log n$ where $\mathrm{med}$ is the median of the pairwise distance between the particles in the current iterate, so that the bandwidth $h$ changes according to the set of particles adaptively across the iterations. As for the KSD estimator, we use the Gaussian kernel $\mathrm{K}(\cdot,\cdot)$ where the bandwidth is set to 0.55. We use AdaGrad for step size and initialize the particles using the prior distribution unless otherwise specified. Finally, we fix the temperature parameter $\alpha=1$ in \eqref{eqn:glo_energy} throughout the iterations.

For exposition, we consider the learning task of training classifiers. Two prevalent learning models of Bayesian logistic regression (BLR) \cite{gershman2012nonparametric} and Bayesian Neural Networks (BNN) \cite{hernandez2015probabilistic} are employed for implementation. First, We consider BLR model for binary classification using the same setting as in \cite{gershman2012nonparametric}. The model parameters is define as $\theta=[\mathbf{w},\log(\xi)]$ where parameter $\xi$ is a precision and the regression weights $\mathbf{w}$ is assigned with a Gaussian prior $p_0(\mathbf{w}|\xi)=\mathcal{N}(\mathbf{w}|\mathbf{0},\xi^{-1}\mathbf{I}_d)$ and $p_0(\xi)=\mathrm{Gamma}(\xi|1,0.01)$. The local likelihood at each client $k$ is given as $p_k=\prod_{x_i,y_i\in D_k}p(y_i|x_i,\mathbf{w},\xi)$, where $D_k$ is the dataset at client $k$ with covariates $x_k\in \mathbb{R}^d$ and label $y_k\in \{-1,1\}$. Similarly, we consider multi-label classification with BNN model using the same setting as in \cite{kassab2022federated}.  The model parameters $\theta=[\mathbf{W},\log(\gamma)]$ include a precision parameter $\gamma$ along with the weights of neural network $\mathbf{W}$ assigned with a Gaussian prior $p_0(\mathbf{W}|\lambda)=\mathcal{N}(\mathbf{W}|\mathbf{0},\lambda^{-1}\mathbf{I}_d)$ with a fixed precision $\lambda=e$. The data of client $k$ is given as $D_k=\{x_i,y_i\}_{i=1}^K$ and $y_n$ is obtained as $y_n=f(x_n; \mathbf{W})+n$ where the neural network output $f(\cdot;\mathbf{W})$ is corrupted by additive noise variables $n$, where $n\sim\mathcal{N}(0,\gamma^{-1})$. The number of hidden layers containing 100 hidden units is set to one and RELU is set as the activation function. The local likelihood at each client $k$ is given as $p_k=\prod_{x_i,y_i\in D_k}p(y_i|x_i,\mathbf{W},\gamma)$.

To better simulate the data distribution, we consider the non-IID data partitioning way as follows. For the BLR model, we choose the Covertype dataset, which contains two types of labels $y\in \{-1,1\}$. We define the number of clients as 30 and 120, where the ratio of the two categories in each client's data is 1:9. For the BNN model, we choose nine typical classes from digit ``0'' to ``8'' in the well-known dataset MNIST for classification and define the number of clients as 27 and 120, where each client owns only three types of dataset, making the data distribution over clients a pathological non-IID manner. Moreover, since a pathological non-IID distribution of clients data, we set the number of local iterations to 10 to ensure that the model converges.

To demonstrate the effectiveness of the proposed selection schemes, we implement two baseline schemes in the following experiments, the round robin selection scheme and the random selection scheme. Specifically, the round robin selection scheme is used by the authors of \cite{kassab2022federated} in DSVGD, which selects each client in a certain order, and the random selection scheme is the more common selection scheme in federated learning, where the server randomly selects a client at each iteration. In addition, since the communication load caused by HIP-based selection scheme is similar to that of the parallel DSVGD algorithm, we also conduct experiment to compare the model performance between the two schemes.

\subsection{KSD-Based Selection Scheme}
\begin{figure}[htb]
  \centering 
  \hspace{-1cm}
  \subfigure[30 clients]{ 
    \includegraphics[width=4.9cm]{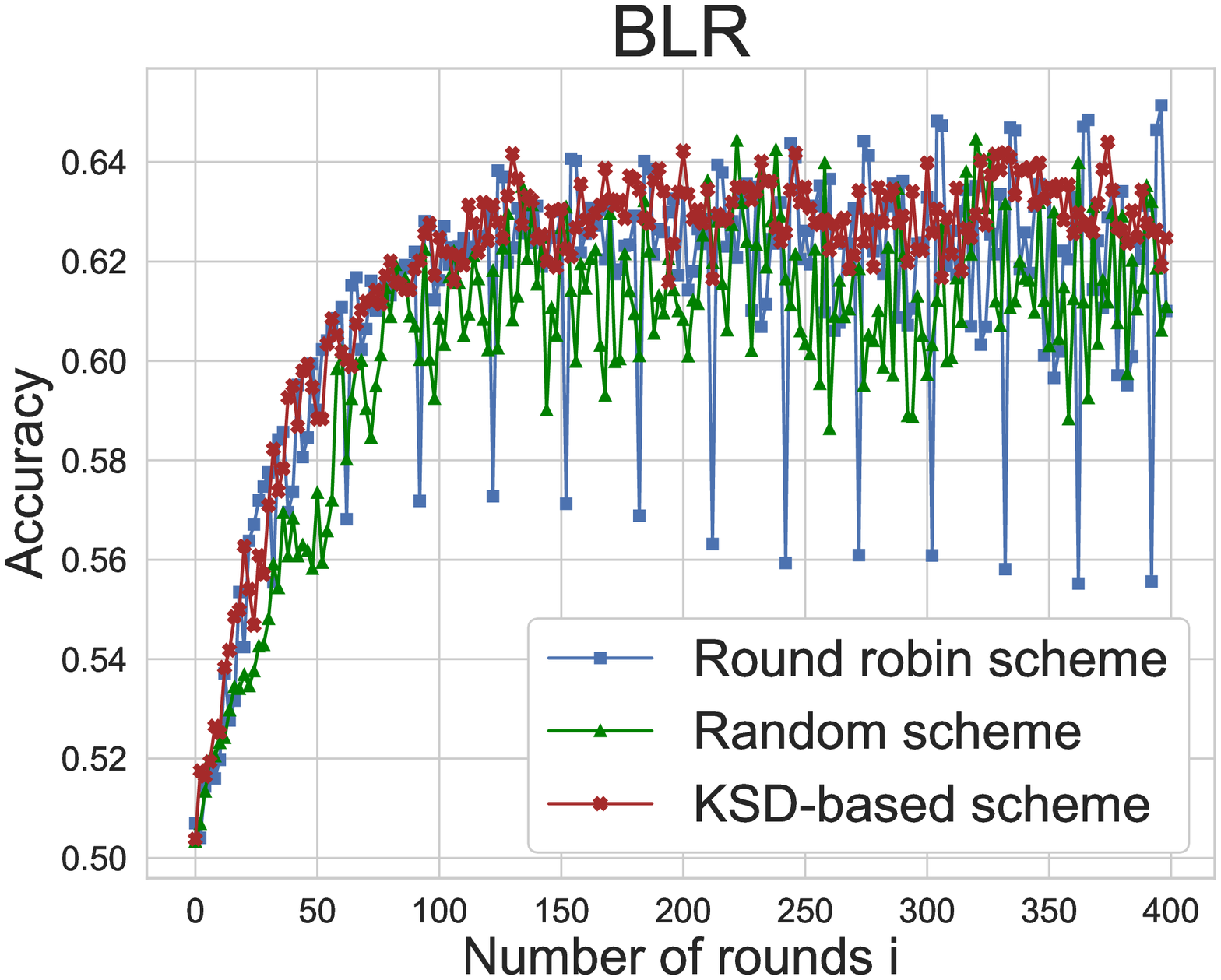} 
  } 
  \hspace{-0.8cm}
  \subfigure[120 clients]{ 
    \includegraphics[width=4.9cm]{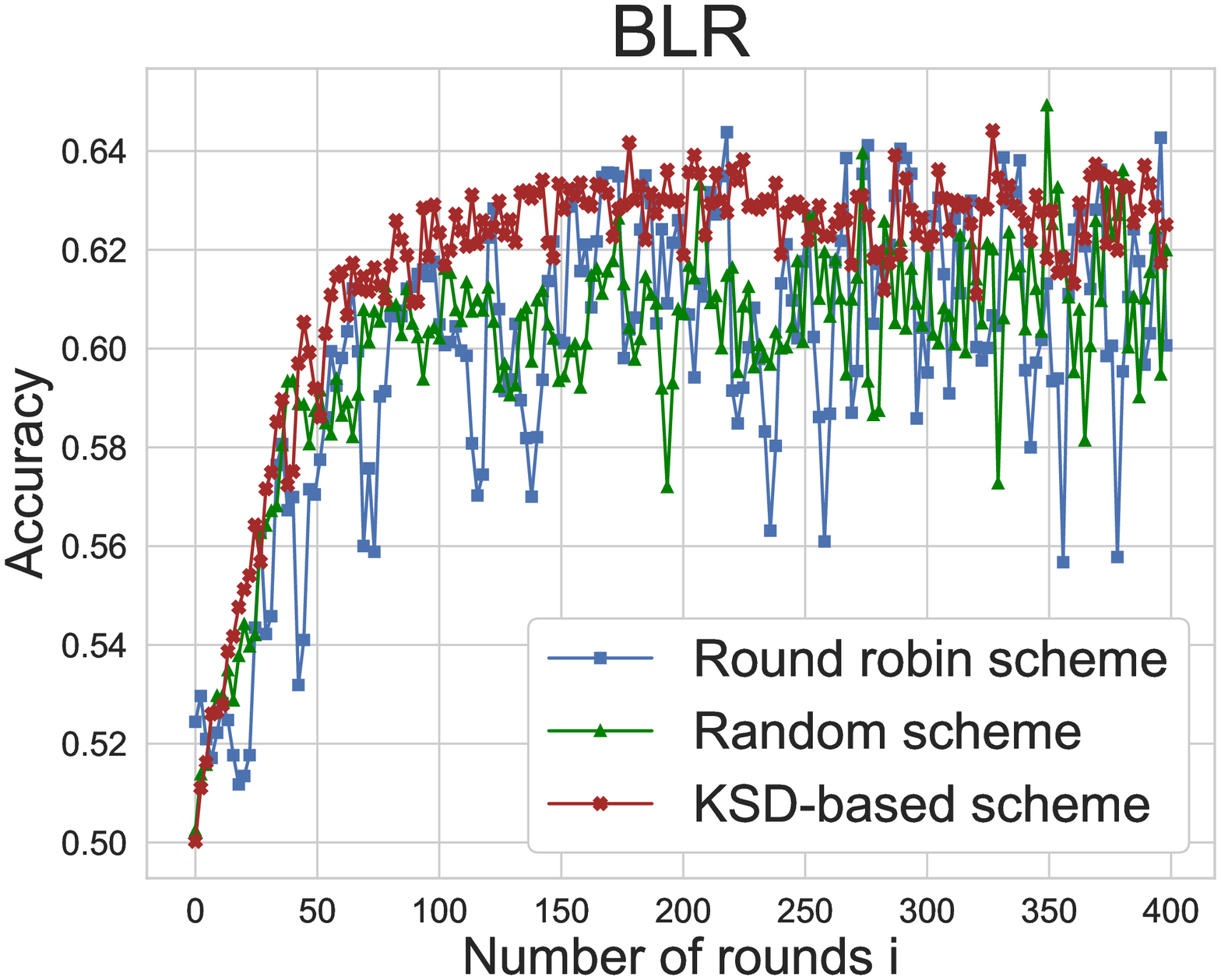} 
  } 
  \hspace{-1cm}
 
  \centering 
  \hspace{-1cm}
  \subfigure[27 clients]{ 
    \includegraphics[width=4.9cm]{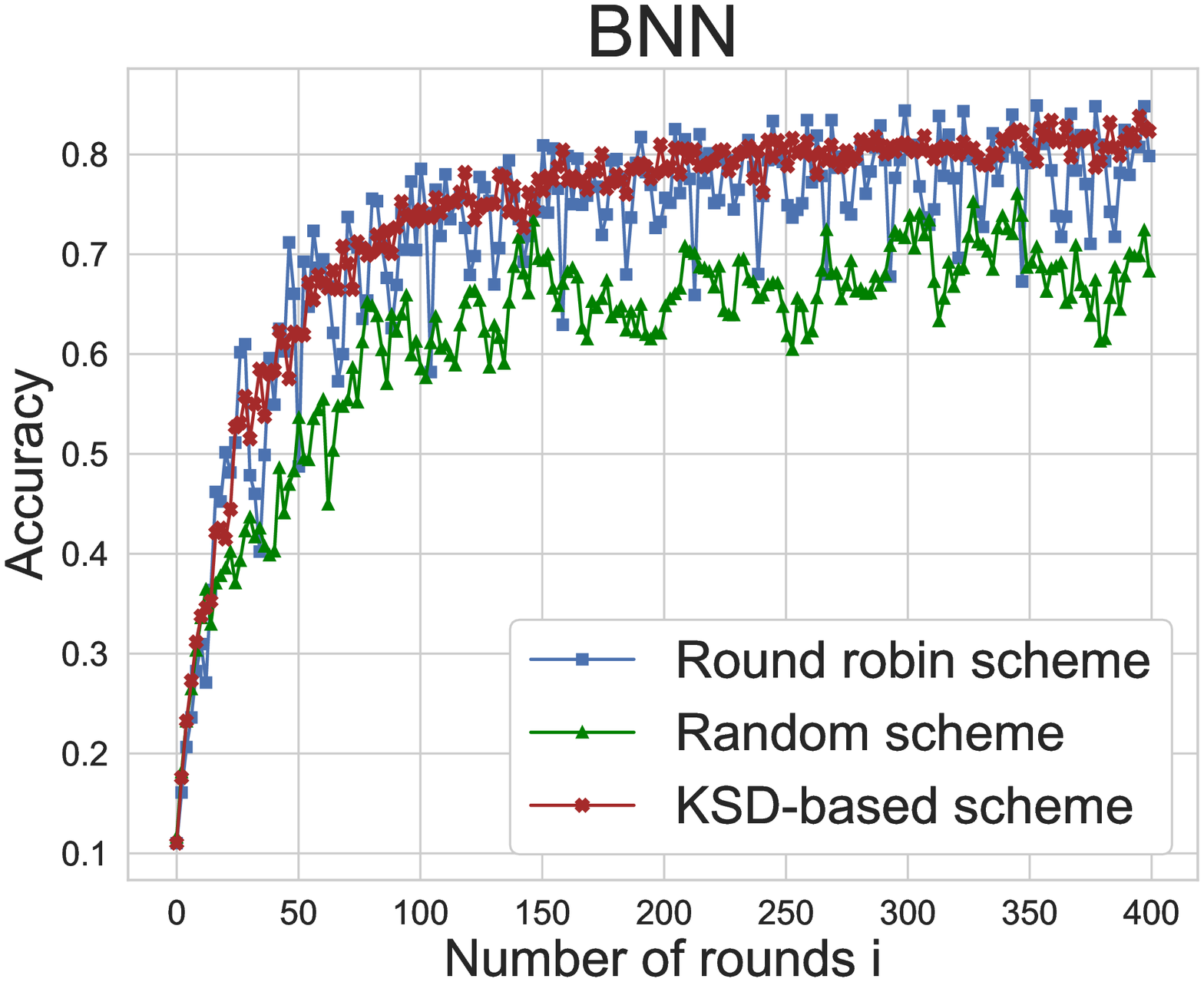} 
  } 
  \hspace{-0.8cm}
  \subfigure[120 clients]{ 
    \includegraphics[width=4.9cm]{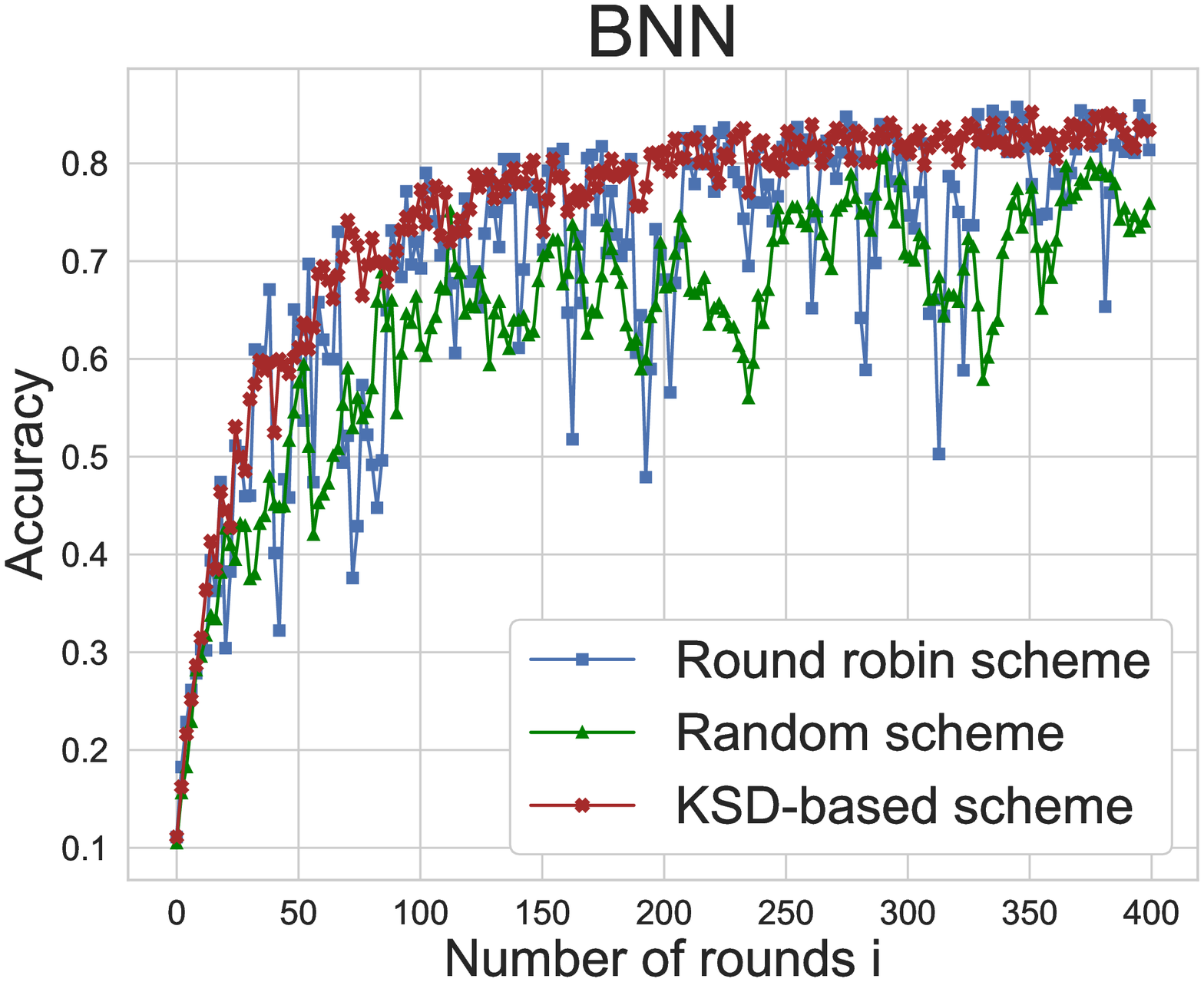} 
  } 
  \hspace{-1cm}
  \caption{Performance comparison among the KSD-based scheme and two baseline schemes.} 
\label{fig:1}
\end{figure}

Fig. \ref{fig:1}(a) and Fig. \ref{fig:1}(b) depicts the learning performance of the KSD-based scheme compared with the two baseline schemes using the BLR model. From the figure, we can observe that the KSD-based scheme achieves the same model convergence speed as the round robin scheme at the early stage of model training, and both schemes converge faster than the random selection scheme. In the later stage of model training, the KSD-based scheme outperforms the other two schemes, where the random scheme has lower model accuracy and worse model convergence, and the model accuracy curve of the round robin scheme vibrates periodically and substantially making the model have no tendency to converge. Fig. \ref{fig:1}(c) and Fig. \ref{fig:1}(d) illustrates the performance comparison among the three schemes using the BNN model. It can be observed that the model performance of the random selection scheme is the worst, and its model accuracy and convergence are inferior to the other two schemes, while the model performance of the KSD-based scheme is better than that of the round robin scheme. Specifically,  with the KSD-based scheme, the convergence curve is more stable than that of the round robin scheme throughout the model training process.  These two experimental results demonstrate the effectiveness of the KSD-based scheme, which achieves the maximization of the decrease of the local free energy.

\subsection{HIP-Based Scheme}
\begin{figure}[htb]
  \centering 
  \hspace{-1cm}
  \subfigure[30 clients]{ 
    \includegraphics[width=4.9cm]{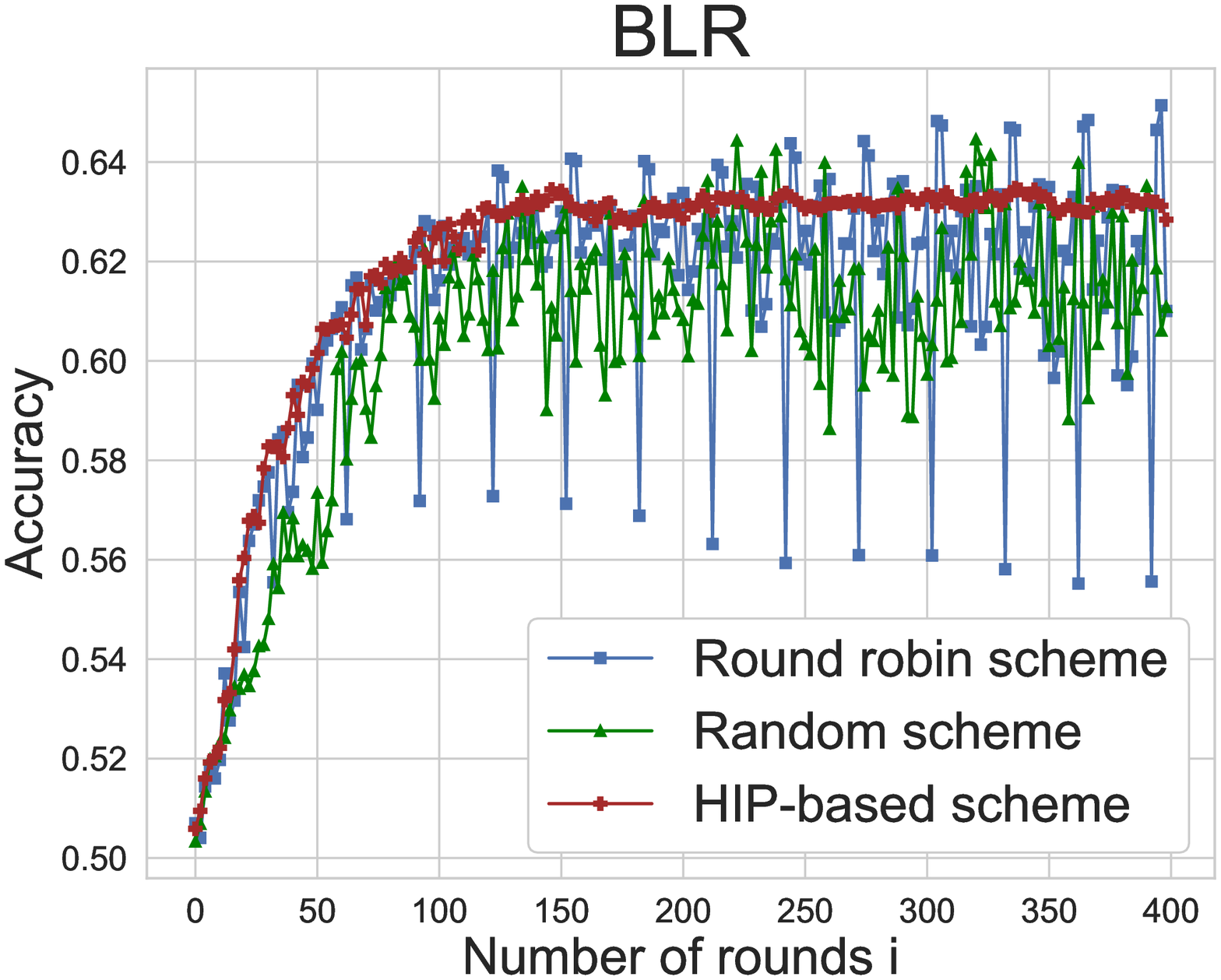} 
  } 
  \hspace{-0.8cm}
  \subfigure[120 clients]{ 
    \includegraphics[width=4.9cm]{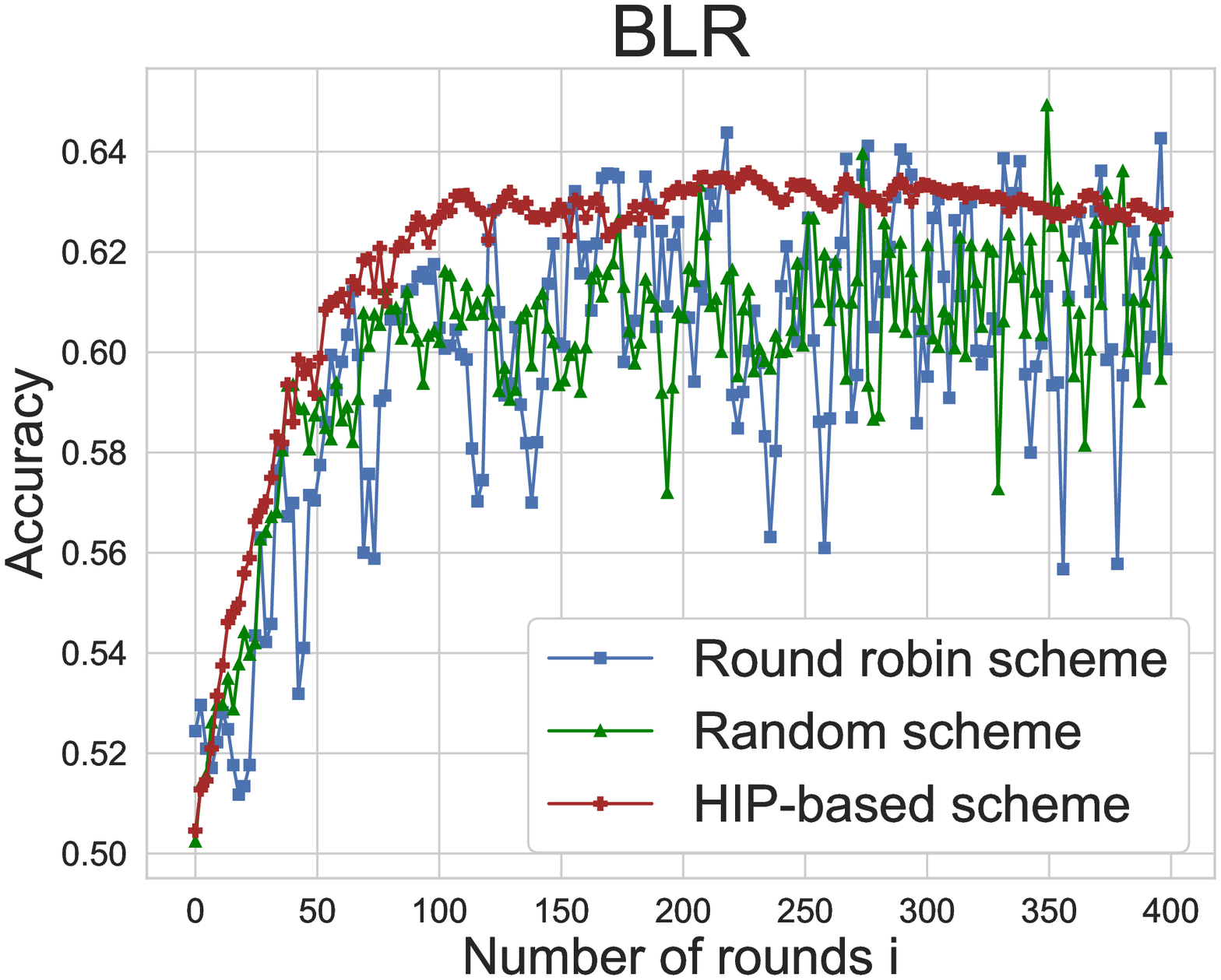} 
  } 
  \hspace{-1cm}
 
  \centering 
  \hspace{-1cm}
  \subfigure[27 clients]{ 
    \includegraphics[width=4.9cm]{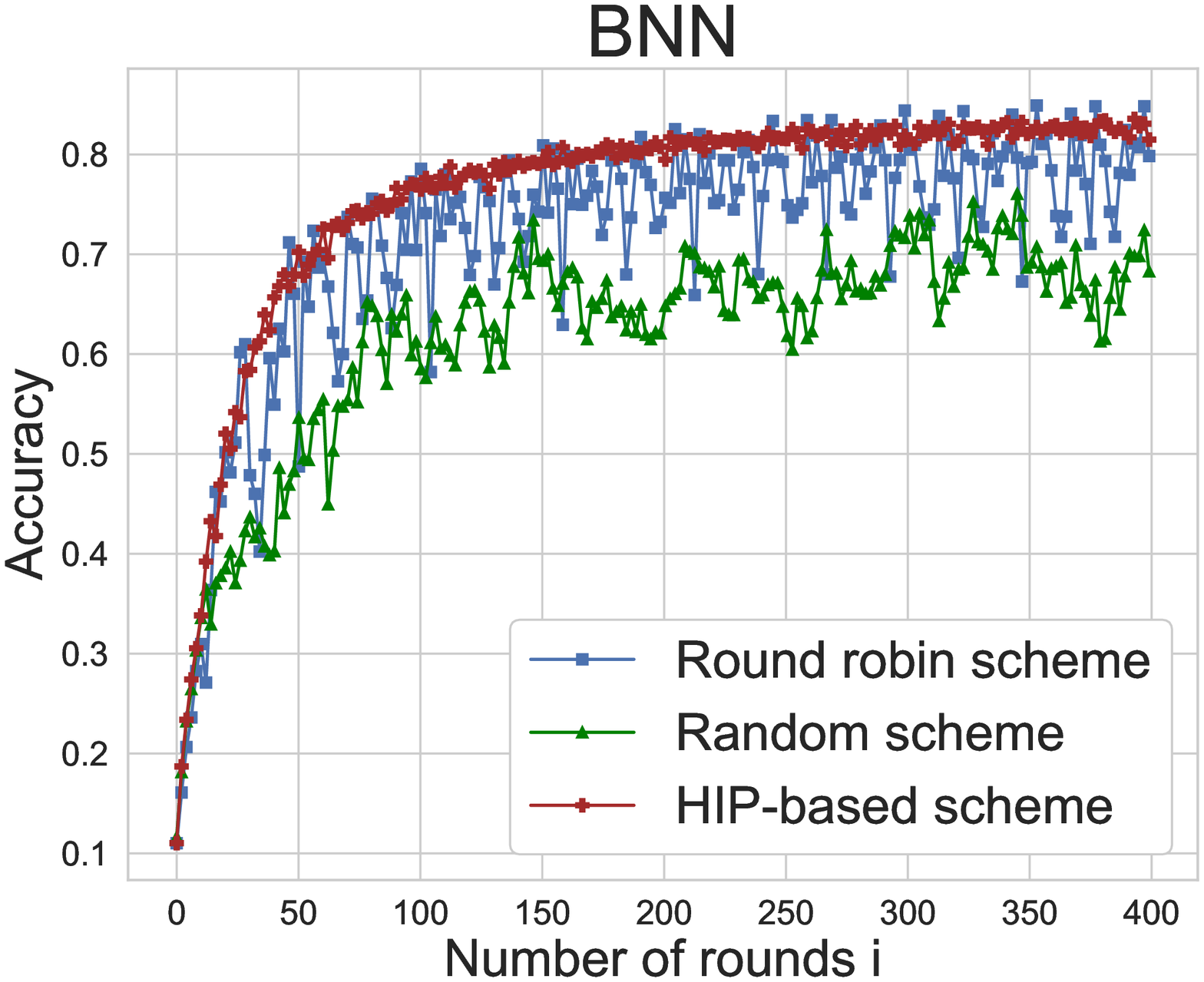} 
  } 
  \hspace{-0.8cm}
  \subfigure[120 clients]{ 
    \includegraphics[width=4.9cm]{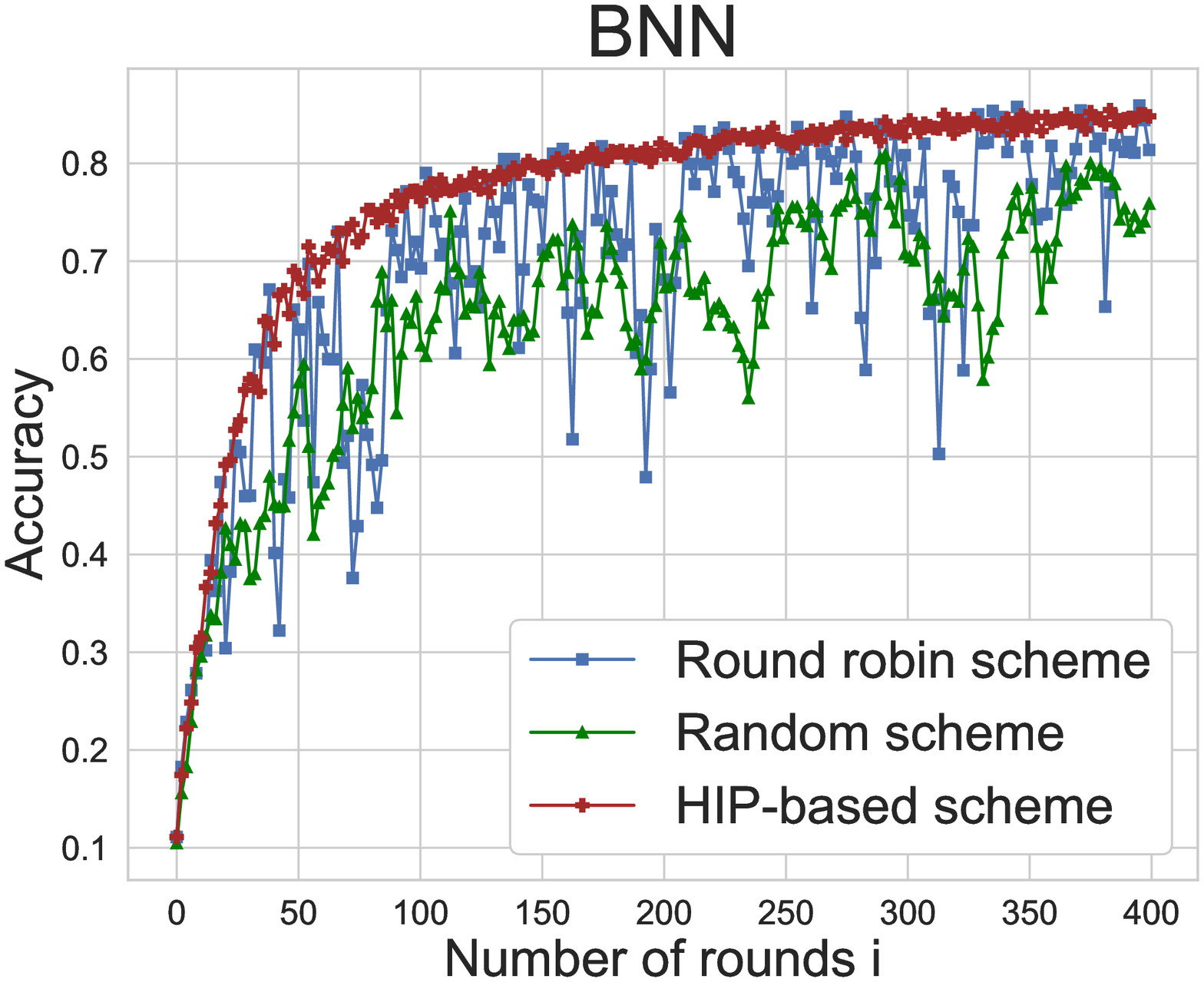} 
  } 
  \hspace{-1cm}
  \caption{Performance comparison among the HIP-based scheme and two baseline schemes.} 
\label{fig:2}
\end{figure}

From Fig. \ref{fig:2}, we observe the better performance of HIP-based scheme compared to that of two baseline schemes using the BLR model and BNN model. Specifically, with HIP-based scheme, the model accuracy curve is always stable than that of two baseline schemes and the model accuracy is almost always higher than that of two baseline schemes, which especially significant in late iterations, demonstrating that  the HIP-based scheme, which aims at minimizing the upper bound of the decrease of the global free energy, significantly improves the model performance in an environment with a pathological non-IID client data distribution, although the scheme causes more communication resources per iteration.

\subsection{Comparison of the Two Proposed Schemes}
\begin{figure}[htb]
  \centering 
  \hspace{-1cm}
  \subfigure[30 clients]{ 
    \includegraphics[width=4.9cm]{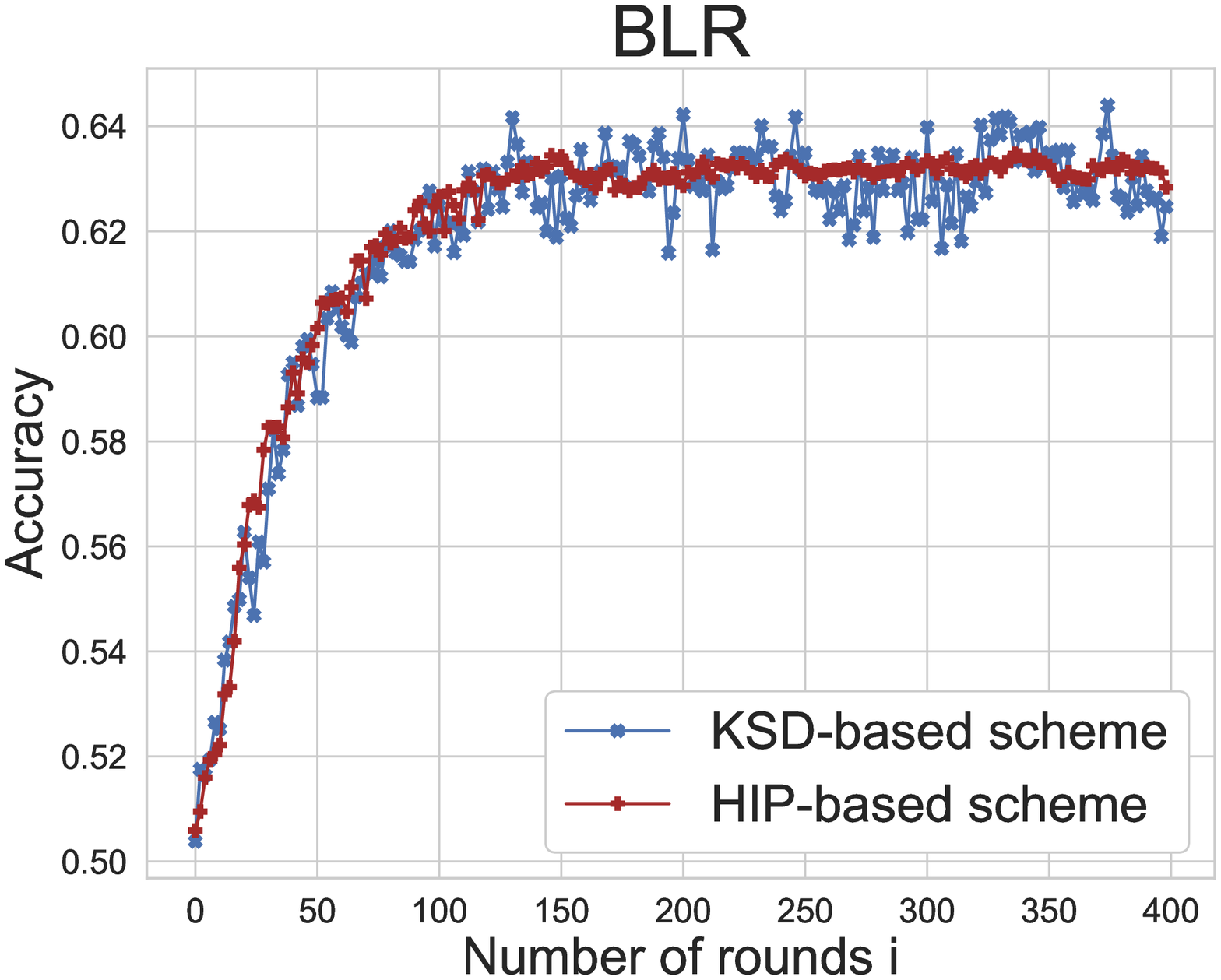} 
  } 
  \hspace{-0.8cm}
  \subfigure[120 clients]{ 
    \includegraphics[width=4.9cm]{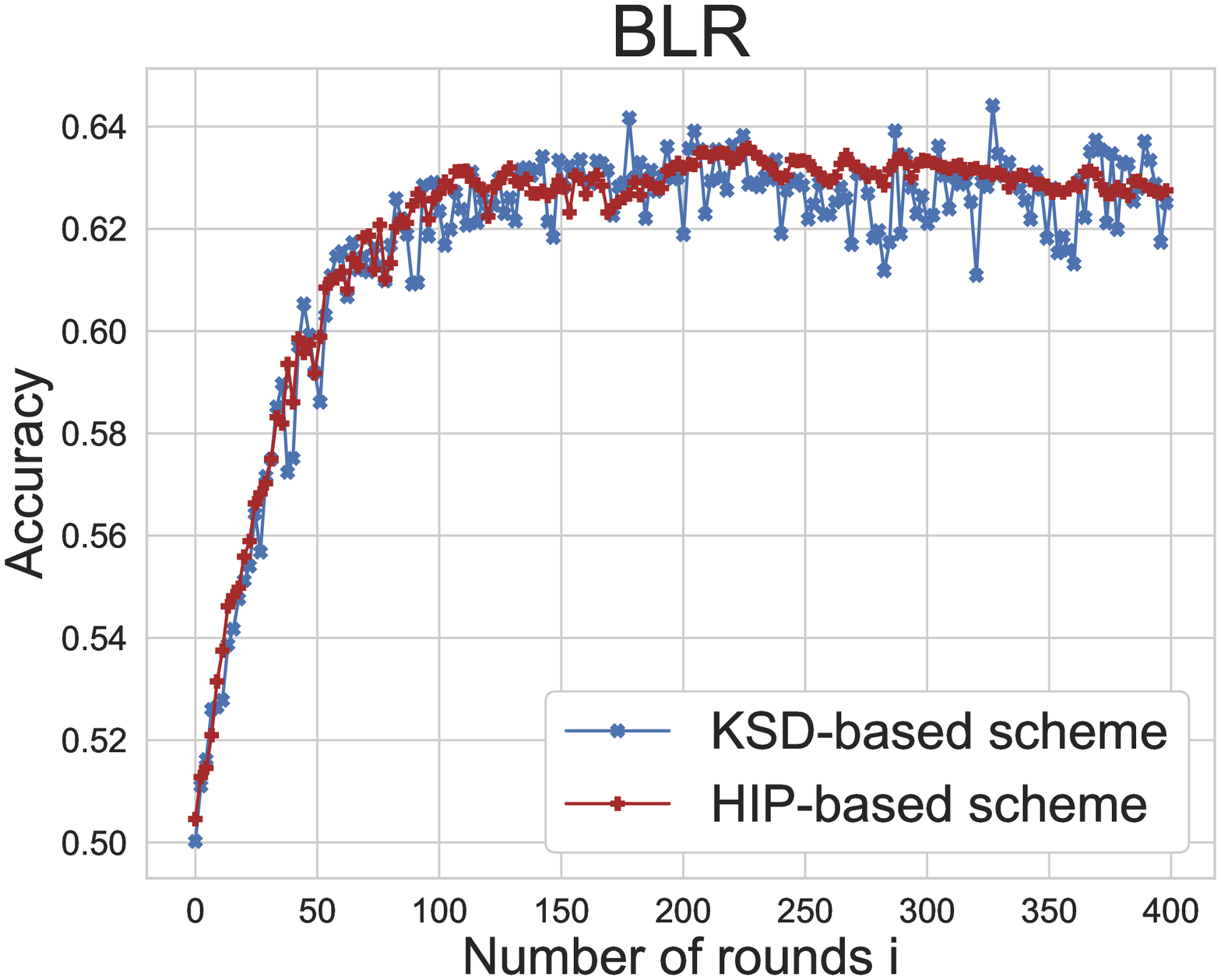} 
  } 
  \hspace{-1cm}
 
  \centering 
  \hspace{-1cm}
  \subfigure[27 clients]{ 
    \includegraphics[width=4.9cm]{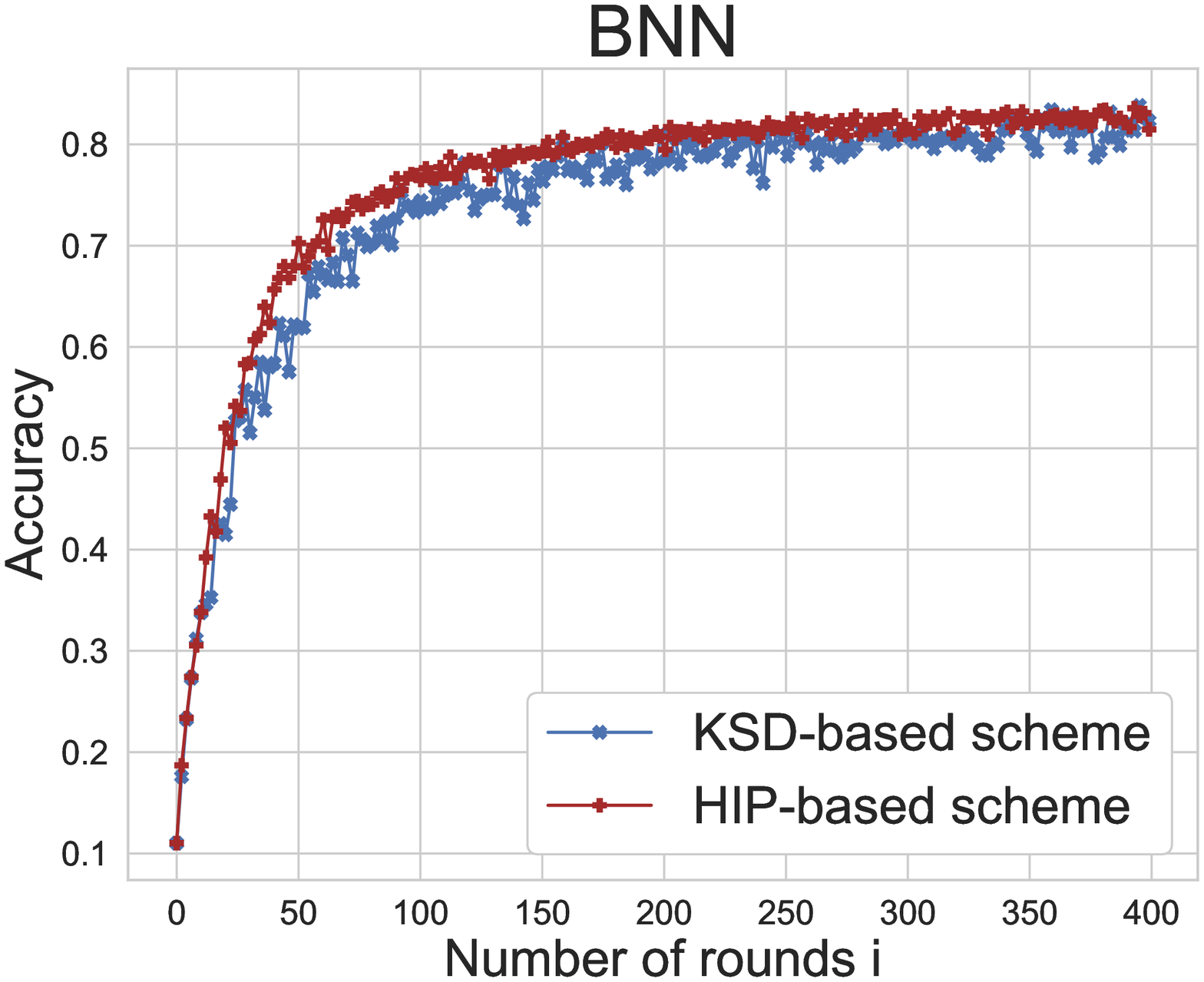} 
  } 
  \hspace{-0.8cm}
  \subfigure[120 clients]{ 
    \includegraphics[width=4.9cm]{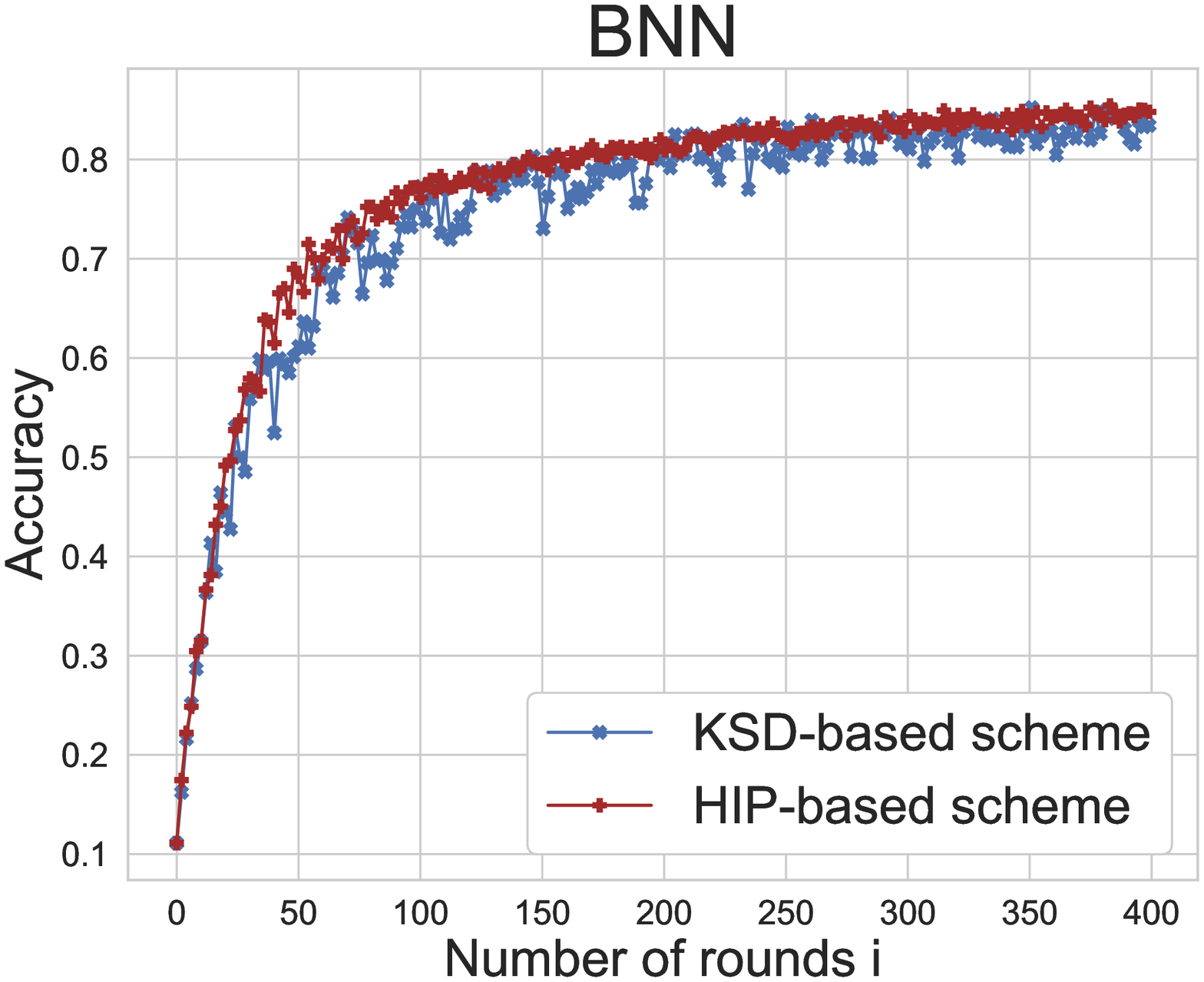} 
  } 
  \hspace{-1cm}
  \caption{Performance comparison between the KSD-based scheme and the HIP-based scheme.} 
\label{fig:3}
\end{figure}

We compare the model performance of the two proposed schemes in the BLR model and the BNN model. From Fig. \ref{fig:3}, we observed that the performance of the HIP-based scheme is better, specifically, the convergence curve of the HIP-based scheme is more stable, while the convergence speed of the HIP-based scheme is slightly faster in the BNN model. This is consistent with the previous theoretical analysis that the KSD-based scheme considered only maximizing the decrease of the local free energy, while the HIP-based scheme maximizes the decrease of the global free energy and thus should exhibit better model performance.

\subsection{Comparison with the Parallel DSVGD Algorithm}
\begin{figure}[htb]
\begin{minipage}[b]{1.0\linewidth}
  \centering
  \centerline{\includegraphics[width=10.5cm]{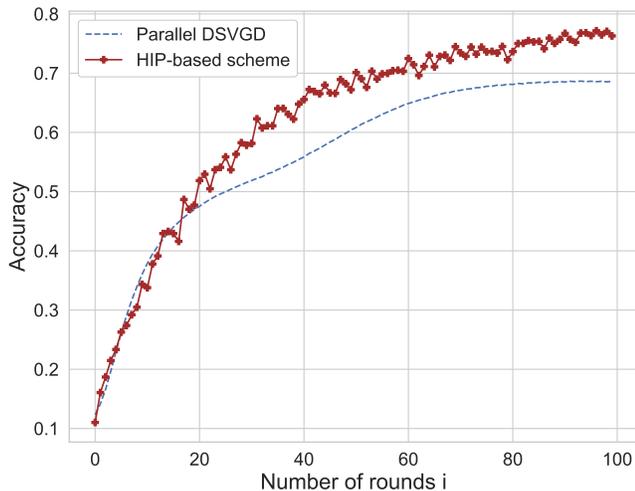}}
\end{minipage}
\caption{Performance comparison between the HIP-based and the parallel DSVGD algorithm using the BNN model.}
\label{fig:4}
\end{figure}

Fig. \ref{fig:4} depicts the performance of the HIP-based scheme compared with the parallel DSVGD algorithm using the BNN model. From this figure, we can observe that the convergence curve of the parallel DSVGD algorithm is more stable than that of the the HIP-based scheme, which is because the parallel DSVGD algorithm can use the likelihood information of all clients when updating the global particles, so that the global posterior does not favor a certain local likelihood. However, as the training proceeds, the model accuracy of the HIP-based scheme is asymptotically higher than that of the parallel DSVGD algorithm. The reason is that in the parallel DSVGD algorithm, there is an error in the KDE of the local particles uploaded at the client side, which does not accurately represent the information of the approximate scaled local likelihood, thus causing the loss of client data information in this step and degrading the accuracy of the model.
\section{Conclusion}
\label{sec:Conclusion}
In this paper, we proposed two client selection schemes for federated Bayesian learning, namely the KSD-based scheme and the HIP-based scheme. Through convergence analysis, we found that the KSD-based scheme maximizes the decrease of the local free energy per iteration and the HIP-based scheme improves the model convergence by maximizing the decrease of the global free energy per iteration. Comprehensive experiments using real datasets confirmed the performance gain of the proposed schemes compared with the baseline schemes.

\section*{Appendix}

\appendices
\subsection{Proof of Lemma \ref{lemma1}}
The decrease of the global free energy per iteration can be written as
\begin{align}\label{eqn:de_glo}
\Big(F(q^{[l+1]}&(\theta))  -F(q^{[l]}(\theta))\Big)/\alpha \nonumber\\
&=\underbrace{\Big(F_k(q^{[l+1]}(\theta))-F_k(q^{[l]}(\theta))\Big)/\alpha}_{(a)} & \nonumber\\
&+\sum_{m \neq k}\underbrace{\int \log \frac{t_m^{(i-1)}(\theta)}{p_m(\theta)} (q^{[l+1]}(\theta)-q^{[l]}(\theta))\,\mathrm{d}\theta}_{(b)}.
\end{align}
We now derive upper bounds for $(a)$ and $(b)$. $(a)$ represents the decrease of the local free energy, We first rewrite $(a)$ in the form of $\mathrm{KL}$ divergence as follows
\begin{align}
(a)= \mathrm{KL}(q^{[l+1]}\|\tilde{p}_k) - \mathrm{KL}(q^{[l]}\|\tilde{p}_k).
\end{align}
SVGD maps particles using $\theta^{[l+1]}=T(\theta^{[l]})=\theta^{[l]}+\epsilon\phi(\theta^{[l]})$, so we have that
\begin{align}\label{eqn:a_kl}
(a) &= \mathrm{KL}(T(q^{[l]})\|\tilde{p}_k) - \mathrm{KL}(q^{[l]}\|\tilde{p}_k) \nonumber\\
&= \mathrm{KL}(q^{[l]}\|T^{-1}(\tilde{p}_k)) - \mathrm{KL}(q^{[l]}\|\tilde{p}_k) \nonumber\\
&= \mathbb{E}_{q^{[l]}(\theta)}[\log q^{[l]}(\theta)- \log T^{-1}(\tilde{p}_k)(\theta)] \nonumber\\
&~~~~-\mathbb{E}_{q^{[l]}(\theta)}[\log q^{[l]}(\theta) - \log \tilde{p}_k(\theta)] \nonumber\\
&= \mathbb{E}_{q^{[l]}(\theta)}[\log \tilde{p}_k(\theta) - \log T^{-1}(\tilde{p}_k)(\theta)].
\end{align}
According to the change of variable formula for densities $T^{-1}(p)(\theta)= p(T(\theta)) \mid \det (\nabla_{\theta} T(\theta)) \mid $, we rewrite \eqref{eqn:a_kl} as
\begin{align}\label{eqn:a_kl_2}
(a) = \mathbb{E}_{q^{[l]}}[\log \tilde{p}_k(\theta) - \log \tilde{p}_k(T(\theta)) - \log \mid\det (\nabla_{\theta} T(\theta))\mid].
\end{align}
We can further simplify the term $\log \tilde{p}_k(\theta) - \log \tilde{p}_k(T(\theta))$ using the Taylor expansion as
\begin{align}\label{eqn:taylor}
\log \tilde{p}_k(\theta) - \log \tilde{p}_k(T(\theta)) \approx &-\epsilon \nabla_{\theta} \log \tilde{p}_k (\theta)^{T} \phi_k (\theta) \nonumber\\
&+ \frac{1}{2}\epsilon^2\nabla_{\theta}^2\log \tilde{p}_k (\theta)\phi_k^2(\theta).
\end{align}
Assume that the term $\nabla_{\theta} T(\theta)$ is a positive definite matrix, which can hold when $\epsilon$ takes a small value such that $1+\epsilon a_i >0$, where $a_i$ is the eigenvalues of $\nabla_{\theta} \phi(\theta)$.  We can upper bound the term $ -\log \mid\det (\nabla_{\theta} T(\theta))\mid$ as
\begin{align}\label{eqn:logdet}
-\log \mid\det (\nabla_{\theta} T(\theta))\mid &\leq -\sum_{i=1}^d(1-e_i^{-1})\nonumber\\
&= \mathrm{trace}((\nabla_{\theta} T(\theta))^{-1}-I),
\end{align}
where $e_1,...,e_d$  are the eigenvalues of $\nabla_{\theta} T(\theta)$. Note that $T(\theta)=\theta+\epsilon\phi(\theta)$, so we have $\nabla_{\theta} T(\theta)=I+\epsilon\nabla_{\theta} \phi(\theta)$, By Neumann expansion, we can obtain the approximation term of $(\nabla_{\theta} T(\theta))^{-1}$ as 
\begin{align}\label{eqn:inverse_T}
(\nabla_{\theta} T(\theta))^{-1} &= (I-(-\epsilon\nabla_{\theta}\phi(\theta)))^{-1}\nonumber\\
&\approx I - \epsilon\nabla_{\theta}\phi(\theta) + (\epsilon\nabla_{\theta}\phi(\theta))^2.
\end{align}
Then we can simplify \eqref{eqn:logdet} using \eqref{eqn:inverse_T} as
\begin{align}\label{eqn:logdet_2}
-\log \mid\det (\nabla_{\theta} T(\theta))\mid \leq \mathrm{trace}(- \epsilon\nabla_{\theta}\phi(\theta) + (\epsilon\nabla_{\theta}\phi(\theta))^2).
\end{align}
Accordingly, we can rewrite \eqref{eqn:a_kl_2} using \eqref{eqn:taylor} and \eqref{eqn:logdet_2} as
\begin{align}\label{eqn:a_kl_3}
(a) \leq &-\epsilon \mathbb{E}_{q^{[l]}}[\mathrm{trace}(\nabla_{\theta} \log \tilde{p}_k(\theta)^T \phi_k(\theta) + \nabla_{\theta} \phi_k(\theta))] \nonumber \\
&+\epsilon^2\mathbb{E}_{q^{[l]}}[\mathrm{trace}((\nabla_{\theta} \phi_k(\theta))^2+\frac{1}{2}\nabla_{\theta}^2\log \tilde{p}_k (\theta)\phi_k^2(\theta))] \nonumber \\
    \leq &-\epsilon S(q^{[l]}, \tilde{p}_k) + \epsilon^2\mathbb{E}_{q^{[l]}}[\mathrm{trace}((\nabla_{\theta} \phi_k(\theta))^2 \nonumber\\
    &+\frac{1}{2}\nabla_{\theta}^2\log \tilde{p}_k (\theta)\phi_k^2(\theta))] .
\end{align}
Here we assume a small learning rate $\epsilon$ such that the second term in \eqref{eqn:a_kl_3} is approximately equal to 0, therefore we obtain the upper bound of $(a)$ as
\begin{align}
(a)\leq -\epsilon S(q^{[l]}, \tilde{p}_k) .
\end{align}
As for the upper bound of $(b)$, we refer to the derivation of \cite{kassab2022federated} as
\begin{align}\label{eqn:b_kl}
(b) \leq 2(K-1)l_{max}^{(i)}\sqrt{2\mathrm{KL}(q^{[l+1]}\|q^{[l]}) },
\end{align}
where $l_{max}^{(i)}=\sup_{\theta} \max_{m \neq k} \mid \log (t_m^{(i-1)}(\theta))\cdot p_m(\theta) \mid $. In summary, the decrease of the global free energy  can be upper bounded as 
\begin{align}
\Big(F(q^{[l+1]}(\theta))&-F(q^{[l]}(\theta))\Big)/\alpha \leq \nonumber\\
-\epsilon S(q^{[l]},& \tilde{p}_k) +2(K-1)l_{max}^{(i)} \sqrt{2\mathrm{KL}(q^{[l+1]}\rVert q^{[l]})}.
\end{align}

\subsection{Proof of Theorem \ref{theorem1}}
Given the variational distribution $q(\theta)$, the target distribution $p_1(\theta)$ and the target distribution $p_2(\theta)$, we can obtain the score functions of $p_1(\theta)$ as $S_{p1} = \nabla_{\theta}\log p_1(\theta)$ and the score functions of $p_2(\theta)$ as $S_{p2} = \nabla_{\theta}\log p_2(\theta)$. Following the SVGD paper, we can obtain the SVGD update functions $\phi_1 = \mathbb{E}_{\theta \sim q}[S_{p1}(\theta)\mathrm{k}(\theta,\cdot)+\nabla_{\theta}\mathrm{k}(\theta,\cdot)]$ and $\phi_2 = \mathbb{E}_{\theta \sim q}[S_{p2}(\theta)\mathrm{k}(\theta,\cdot)+\nabla_{\theta}\mathrm{k}(\theta,\cdot)]$. In the Hilbert space defined by kernel $\mathrm{k}(\theta,\theta')$, we have the HIP between the two update functions as

\begin{align}\label{eqn:hip_1}
\langle \phi_1, \phi_2 \rangle_{\mathcal{H}^d} = \sum_{l=1}^d\langle &\phi_1, \phi_2 \rangle_{\mathcal{H}} \nonumber \\
=\sum_{l=1}^d \langle &\mathbb{E}_{\theta \sim q}[S_{p1}^l(\theta)\mathrm{k}(\theta,\cdot)+\nabla_{\theta}\mathrm{k}(\theta,\cdot)], \nonumber\\
&\mathbb{E}_{\theta' \sim q}[S_{p2}^l(\theta')\mathrm{k}(\theta',\cdot)+\nabla_{\theta'}\mathrm{k}(\theta',\cdot)] \rangle_{\mathcal{H}} \nonumber \\
\end{align}
According to the Stein's Identity, we can obtain that
\begin{align}
\mathbb{E}_{\theta \sim q}[S_{q}(\theta)\mathrm{k}(\theta,\cdot)^T + \nabla_{\theta}\mathrm{k}(\theta,\cdot)] = 0,
\end{align}
therefore, we can rewrite the term $\mathbb{E}_{\theta \sim q}[S_{p}(\theta)\mathrm{k}(\theta,\cdot)^T + \nabla_{\theta}\mathrm{k}(\theta,\cdot)] $ as 
\begin{align}\label{eqn:stein_identity}
\mathbb{E}_{\theta \sim q}[S_{p}(\theta)\mathrm{k}(\theta,\cdot)^T &+ \nabla_{\theta}\mathrm{k}(\theta,\cdot)]  \nonumber\\
=& \mathbb{E}_{\theta \sim q}[S_{p}(\theta)\mathrm{k}(\theta,\cdot)^T + \nabla_{\theta}\mathrm{k}(\theta,\cdot) \nonumber\\
&- S_{q}(\theta)\mathrm{k}(\theta,\cdot)^T - \nabla_{\theta}\mathrm{k}(\theta,\cdot) ]  \nonumber \\
=&\mathbb{E}_{\theta \sim q}[(S_p(\theta)-S_q(\theta))\mathrm{k}(\theta,\cdot)^T].
\end{align}
Accordingly, we can further simplify \eqref{eqn:hip_1} using \eqref{eqn:stein_identity} as
\begin{align}\label{eqn:hip_2}
\langle \phi_1 &, \phi_2 \rangle_{\mathcal{H}^d} \nonumber\\
&=\sum_{l=1}^d \langle \mathbb{E}_{\theta \sim q}[(S_{p1}^l(\theta)-S_q^l(\theta))\mathrm{k}(\theta,\cdot)],\nonumber\\
&~~~~~~~~~~~~~~~\mathbb{E}_{\theta' \sim q}[ (S_{p2}^l(\theta')-S_q^l(\theta'))\mathrm{k}(\theta',\cdot)  ]\rangle_{\mathcal{H}} \nonumber \\
&=\mathbb{E}_{\theta,\theta' \sim q}[(S_{p1}(\theta)-S_q(\theta))^T \langle \mathrm{k}(\theta,\cdot) ,\mathrm{k}(\theta',\cdot) \rangle_{\mathcal{H}}\nonumber\\
&~~~~~~~~~~~~~~~~~~~~~~~~~~~~~~~~~~~(S_{p2}(\theta')-S_q(\theta'))  ]\nonumber \\
&=\mathbb{E}_{\theta,\theta' \sim q}[(S_{p1}(\theta)-S_q(\theta))^T\mathrm{k}(\theta,\theta')(S_{p2}(\theta')-S_q(\theta'))].
\end{align}
Here we define $v(\theta,\theta')=\mathrm{k}(\theta,\theta')S_{p2}(\theta')+\nabla_{\theta'}\mathrm{k}(\theta,\theta')$. For the first term $\mathrm{k}(\theta,\theta')S_{p2}(\theta')$, we have that
\begin{align}
\int_{\theta} \nabla_{\theta}(q(\theta)\mathrm{k}(\theta,\theta')&S_{p2}(\theta')) d\theta \nonumber\\
&= S_{p2}(\theta')\int_{\theta} \nabla_{\theta}(q(\theta)\mathrm{k}(\theta,\theta')) d\theta.
\end{align}
Since $\mathrm{k}(\theta,\theta')$ is in the Stein class of q, we have 
\begin{align}
\int_{\theta} \nabla_{\theta}(q(\theta)\mathrm{k}(\theta,\theta')S_{p2}(\theta')) d\theta =  S_{p2}(\theta') \cdot 0 = 0. 
\end{align}
So the first term is in the Stein class of q. The second $\nabla_{\theta'}\mathrm{k}(\theta,\theta')$ is also in the Stein class of q as follows
\begin{align}
\int_{\theta} \nabla_{\theta}(q(\theta)&\nabla_{\theta'}\mathrm{k}(\theta,\theta')) d\theta \nonumber\\
&= \int_{\theta}\nabla_{\theta}q(\theta)\nabla_{\theta'}\mathrm{k}(\theta,\theta')+q(\theta)\nabla_{\theta}\nabla_{\theta'}\mathrm{k}(\theta,\theta')d\theta \nonumber \\
&=\nabla_{\theta'} \int_{\theta}\nabla_{\theta}q(\theta)\mathrm{k}(\theta,\theta')+q(\theta)\nabla_{\theta}\mathrm{k}(\theta,\theta')d\theta \nonumber \\
&=\nabla_{\theta'} \int_{\theta} \nabla_{\theta}(q(\theta)\mathrm{k}(\theta,\theta'))d\theta \nonumber \\
&=0.
\end{align}
So $v(\theta,\theta')$ is still in the Stein class of q and we have $v(\theta,\theta')=\mathrm{k}(\theta,\theta')S_{p2}(\theta')+\nabla_{\theta'}\mathrm{k}(\theta,\theta') = \mathrm{k}(\theta,\theta')(S_{p2}(\theta')-S_{q}(\theta'))$ and further simplify \eqref{eqn:hip_2} as
\begin{align}\label{eqn:hip_3}
&=\mathbb{E}_{\theta,\theta' \sim q}[(S_{p1}(\theta)-S_q(\theta))^T v(\theta,\theta')] \nonumber \\
&=\mathbb{E}_{\theta,\theta' \sim q}[S_{p_1}(\theta)^Tv(\theta,\theta') + \nabla_{\theta}v(\theta,\theta')] \nonumber \\
&=\mathbb{E}_{\theta,\theta' \sim q}[S_{p_1}(\theta)^T(\mathrm{k}(\theta,\theta')S_{p2}(\theta')+\nabla_{\theta'}\mathrm{k}(\theta,\theta'))\nonumber\\
&~~~~~~~~~+\nabla_{\theta}(\mathrm{k}(\theta,\theta')S_{p2}(\theta')+\nabla_{\theta'}\mathrm{k}(\theta,\theta'))] \nonumber \\
&=\mathbb{E}_{\theta,\theta' \sim q}[S_{p_1}(\theta)^T\mathrm{k}(\theta,\theta')S_{p2}(\theta')+ S_{p_1}(\theta)^T\nabla_{\theta'}\mathrm{k}(\theta,\theta')\nonumber \\
&~~~~~~~~~+ \nabla_{\theta}\mathrm{k}(\theta,\theta')^T S_{p2}(\theta') + \mathrm{trace}(\nabla_{\theta,\theta'}\mathrm{k}(\theta,\theta'))].
\end{align}
Here we define
\begin{align}\label{eqn:h_p1p2}
h_{p_1,p_2}(\theta,\theta')=&S_{p_1}(\theta)^T\mathrm{k}(\theta,\theta')S_{p2}(\theta')+ S_{p_1}(\theta)^T\nabla_{\theta'}\mathrm{k}(\theta,\theta')\nonumber \\
&+ \nabla_{\theta}\mathrm{k}(\theta,\theta')^T S_{p2}(\theta') + \mathrm{trace}(\nabla_{\theta,\theta'}\mathrm{k}(\theta,\theta')).
\end{align}
Accordingly we can rewrite \eqref{eqn:hip_3} using \eqref{eqn:h_p1p2} as
\begin{align}\label{eqn:hip_5}
\langle \phi_1, \phi_2 \rangle_{\mathcal{H}^d} = \mathbb{E}_{\theta,\theta' \sim q}[h_{p_1,p_2}(\theta,\theta')].
\end{align}
Given i.i.d. sample $\{\theta\}_{i=1}^N$ drawn from $q(\theta)$, we can estimate $\langle \phi_1, \phi_2 \rangle_{\mathcal{H}^d}$ using a V-statistic of form $\langle \phi_1, \phi_2 \rangle_{\mathcal{H}^d} = \frac{1}{N^2}\sum_{i,j=1}^N[h_{p_1,p_2}(\theta_i,\theta_j)]$.

\subsection{Proof of Lemma \ref{lemma2}}
According to Lemma 1, the decrease of the global free energy is rewritten as \eqref{eqn:de_glo} and divided into $(a)$ and $(b)$. Here we re-derive the upper bound for $(a)$ as
\begin{align}\label{eqn:a_le2_1}
(a) =& \left(F_k(q^{[l+1]}(\theta))-F_k(q^{[l]}(\theta))\right)/\alpha \nonumber \\
    =& \mathbb{E}_{q^{[l+1]}}\left[\log \frac{q^{[l+1]}(\theta)}{p_k(\theta)\cdot(p_0(\theta)\prod_{m\neq k}t_m^{(i-1)}(\theta))} \right] \nonumber \\
    &- \mathbb{E}_{q^{[l]}}\left[\log \frac{q^{[l]}(\theta)}{p_k(\theta)\cdot(p_0(\theta)\prod_{m\neq k}t_m^{(i-1)}(\theta))} \right]\nonumber \\
    =& \mathbb{E}_{q^{[l+1]}}\left[\log \frac{q^{[l+1]}(\theta)}{p_k(\theta)}\right] - \mathbb{E}_{q^{[l]}}\left[\log \frac{q^{[l]}(\theta)}{p_k(\theta)}\right]  \nonumber \\ &+\mathbb{E}_{q^{[l+1]}}\left[\log \frac{1}{p_0(\theta)\prod_{m\neq k}t_m^{(i-1)}(\theta)} \right] \nonumber \\
    &- \mathbb{E}_{q^{[l]}}\left[\log \frac{1}{p_0(\theta)\prod_{m\neq k}t_m^{(i-1)}(\theta)} \right]\nonumber \\
    =& \mathrm{KL}(q^{[l+1]}\|p_k) - \mathrm{KL}(q^{[l]}\|p_k)  \nonumber \\
    &+\mathbb{E}_{q^{[l+1]}}\left[\log \frac{1}{p_0(\theta)\prod_{m\neq k}t_m^{(i-1)}(\theta)} \right] \nonumber \\
    &- \mathbb{E}_{q^{[l]}}\left[\log \frac{1}{p_0(\theta)\prod_{m\neq k}t_m^{(i-1)}(\theta)} \right].
\end{align}
On the rightmost side of \eqref{eqn:a_le2_1}, we replace the first term with (c) and the remaining terms with (d). We first derive the upper bound for $(c)$ using the derivation of Lemma 1 and the equation $S(q^{[l]}, p_k) = \langle \phi_k, \phi_k \rangle_{\mathcal{H}^d}$ as
\begin{align}\label{eqn:c_up}
(c) &\leq -\epsilon S(q^{[l]}, p_k) \nonumber \\
&= -\epsilon \langle \phi_k, \phi_k \rangle_{\mathcal{H}^d} .
\end{align}
We introduce a normalization constant $C=\prod_{m=1}^K C_m$ such that $\int\frac{t_m(\theta)}{C_k} \mathrm{d}\theta = 1$ for $m=1,\ldots, K$. Then we derive the upper bound for $(d)$ as
\begin{align}\label{eqn:d_up}
(d)=& \mathbb{E}_{q^{[l+1]}}\left[\log \frac{t_k^{(i)}(\theta)}{p_0(\theta)t_k^{(i)}(\theta)\prod_{m\neq k}t_m^{(i-1)}(\theta)} \right] \nonumber\\
&- \mathbb{E}_{q^{[l]}}\left[\log \frac{t_k^{(i-1)}(\theta)}{p_0(\theta)t_k^{(i-1)}(\theta)\prod_{m\neq k}t_m^{(i-1)}(\theta)} \right]\nonumber \\
=& \mathbb{E}_{q^{[l+1]}}\left[\log \frac{t_k^{(i)}}{q^{[l+1]}}\right] - \mathbb{E}_{q^{[l]}}\left[\log \frac{t_k^{(i-1)}}{q^{[l]}} \right]\nonumber \\
\leq& \mathrm{KL}(q^{[l]}\|t_k^{(i-1)}) + \log \left(C_k \int \frac{t_k^{(i)}(\theta)}{C_k} \mathrm{d}\theta\right) \nonumber \\
\leq&\mathrm{KL}(q^{[l]}\|t_k^{(i-1)}) + \log C_k.
\end{align}
Therefore, $(a)$ can be upper bounded using \eqref{eqn:c_up} and \eqref{eqn:d_up} as
\begin{align}\label{eqn:a_le2_2}
(a)\leq-\epsilon \langle \phi_k, \phi_k \rangle_{\mathcal{H}^d} +  \mathrm{KL}(q^{[l]}\|t_k^{(i-1)}) + \log C_k.
\end{align}
We now re-derive the upper bound for $(b)$ as
\begin{align}\label{eqn:b_le2_1}
(b) =& \int \log \frac{t_m^{i-1}(\theta)}{p_m(\theta)} (q^{[l+1]}(\theta)-q^{[l]}(\theta))\,\mathrm{d}\theta \nonumber \\
=& \mathbb{E}_{q^{[l+1]}(\theta)}\log \frac{q^{[l+1]}(\theta)}{p_m(\theta)} + \mathbb{E}_{q^{[l+1]}(\theta)}\log \frac{t_m^{(i-1)}(\theta)}{q^{[l+1]}(\theta)}\nonumber\\
&- \mathbb{E}_{q^{[l]}(\theta)}\log \frac{q^{[l]}(\theta)}{p_m(\theta)} - \mathbb{E}_{q^{[l]}(\theta)}\log \frac{t_m^{(i-1)}(\theta)}{q^{[l]}(\theta)}\nonumber \\
\leq& \underbrace{\mathrm{KL}(q^{[l+1]}\|p_m) - \mathrm{KL}(q^{[l]}\|p_m)}_{(e)} \nonumber\\
&+ \mathrm{KL}(q^{[l]}\|t_m^{(i-1)}) + \log C_m.
\end{align}
Here we simplify $(e)$ using the derivation of Lemma 1 as
\begin{align}\label{eqn:b_le2_2}
(e) \leq -\epsilon\mathbb{E}_{q^{[l]}}[\mathrm{trace}(\nabla_{\theta} \log p_m(\theta)^T \phi_k(\theta) + \nabla_{\theta}\phi_k(\theta))] 
\end{align}
Note that $\phi_k(\theta)=\mathbb{E}_{\theta \sim q}[S_{p_k}(\theta)\mathrm{k}(\theta,\cdot)+\nabla_{\theta}\mathrm{k}(\theta,\cdot)]$ is not the SVGD update function of distribution $p_m(\theta)$ . According to the properties of the RKHS, we have
\begin{align}
\mathbb{E}_{q^{[l]}}[\mathrm{trace}(&\nabla_{\theta} \log p_m(\theta)^T \phi_k(\theta)  + \nabla_{\theta}\phi_k(\theta))] \nonumber \\
&= \sum_{l=1}^d  \mathbb{E}_{q^{[l]}}[S_{p_m}^l(\theta)\phi_k^l(\theta)+ \nabla_{\theta}\phi_k^l(\theta)]   \nonumber \\
&= \sum_{l=1}^d  \mathbb{E}_{q^{[l]}}[S_{p_m}^l(\theta)\langle \phi_k^l(\cdot), \mathrm{k}(\theta, \cdot) \rangle_{\mathcal{H}} \nonumber\\
&~~~~~~~~~~~~~~~~~~~~~~~~~~~+ \langle \phi_k^l(\cdot), \nabla_{\theta^l}\mathrm{k}(\theta, \cdot) \rangle_{\mathcal{H}}]\nonumber \\
&= \sum_{l=1}^d  \langle \phi_k^l, \mathbb{E}_{q^{[l]}}[S_{p_m}^l(\theta)\mathrm{k}(\theta,\cdot)+\nabla_{\theta^l}\mathrm{k}(\theta,\cdot)]\rangle_{\mathcal{H}}\nonumber \\
&= \sum_{l=1}^d \langle \phi_k^l, \phi_m^l \rangle_{\mathcal{H}^d}\nonumber \\
&= \langle \phi_k, \phi_m \rangle_{\mathcal{H}^d},
\end{align}
which allows us to rewrite \eqref{eqn:b_le2_2} as
\begin{align}\label{eqn:hip_k_m}
-\epsilon \langle \phi_k, \phi_m \rangle_{\mathcal{H}^d} .
\end{align}
Accordingly, $(b)$ can be upper bounded using \eqref{eqn:b_le2_1} and \eqref{eqn:hip_k_m} as
\begin{align}\label{eqn:b_le2_3}
(b) \leq -\epsilon \langle \phi_k, \phi_m \rangle_{\mathcal{H}^d} + \mathrm{KL}(q^{[l]}\|t_m^{(i-1)})+\log C_m.
\end{align}
In summary, we can obtain the upper bound of the decrease of the global free energy using \eqref{eqn:a_le2_2} and \eqref{eqn:b_le2_3} as follows
\begin{align}
\left(F(q^{[l+1]}(\theta))-F(q^{[l]}(\theta))\right)/\alpha\leq &-\epsilon \sum_{m=1}^K \left \langle \phi_k,\phi_m \right \rangle_{\mathcal{H}^d} \nonumber\\
+ \sum_{m=1}^K \mathrm{KL}&(q^{[l]}\rVert t_m^{(i-1)}) + \log C.
\end{align}

\subsection{Proof of Lemma \ref{lemma3}}
Given the SVGD update function $\phi'$ for the distribution $p'=\left(\prod_{m=1}^K p_m \right)^{\frac{1}{K}}$, we have that
\begin{align}\label{eqn:sum_k_m}
\sum_{m=1}^K &\left \langle \phi_k,\phi_m \right \rangle_{\mathcal{H}^d} \nonumber\\
 =& \sum_{m=1}^K \mathbb{E}_{\theta,\theta' \sim q}[S_{p_k}(\theta)^T \mathrm{k}(\theta,\theta')S_{p_m}(\theta') \nonumber\\
 &+ S_{p_k}(\theta)^T \nabla_{\theta'}\mathrm{k}(\theta,\theta') 
 + \nabla_{\theta}\mathrm{k}(\theta,\theta')^T S_{p_m}(\theta') \nonumber\\
 &+ \mathrm{trace}(\nabla_{\theta,\theta'}\mathrm{k}(\theta,\theta'))] \nonumber \\
 =& \mathbb{E}_{\theta,\theta' \sim q}[S_{p_k}(\theta)^T \mathrm{k}(\theta,\theta')\sum_{m=1}^KS_{p_m}(\theta') \nonumber\\
 &+ K\cdot S_{p_k}(\theta)^T \nabla_{\theta'}\mathrm{k}(\theta,\theta') + \nabla_{\theta}\mathrm{k}(\theta,\theta')^T \sum_{m=1}^KS_{p_m}(\theta') \nonumber\\
&+ K\cdot\mathrm{trace}(\nabla_{\theta,\theta'}\mathrm{k}(\theta,\theta'))]
 \nonumber \\
 =& K\cdot\mathbb{E}_{\theta,\theta' \sim q}[S_{p_k}(\theta)^T \mathrm{k}(\theta,\theta')\left(\frac{1}{K}\sum_{m=1}^KS_{p_m}(\theta')\right) \nonumber\\
 +& S_{p_k}(\theta)^T \nabla_{\theta'}\mathrm{k}(\theta,\theta') + \nabla_{\theta}\mathrm{k}(\theta,\theta')^T \left(\frac{1}{K}\sum_{m=1}^KS_{p_m}(\theta')\right) \nonumber\\
 &+ \mathrm{trace}(\nabla_{\theta,\theta'}\mathrm{k}(\theta,\theta'))]
 \nonumber \\
=& K\cdot\mathbb{E}_{\theta,\theta' \sim q}[S_{p_k}(\theta)^T \mathrm{k}(\theta,\theta')S_{p'}(\theta') \nonumber\\
 &+ S_{p_k}(\theta)^T \nabla_{\theta'}\mathrm{k}(\theta,\theta') + \nabla_{\theta}\mathrm{k}(\theta,\theta')^T S_{p'}(\theta') \nonumber\\
&+ \mathrm{trace}(\nabla_{\theta,\theta'}\mathrm{k}(\theta,\theta'))]
 \nonumber \\
 =& K \cdot \left \langle \phi_k,\phi' \right \rangle_{\mathcal{H}^d}.
\end{align}

\bibliographystyle{IEEEtran}
\bibliography{IEEEabrv,link}

\begin{thebibliography}{10}
\providecommand{\url}[1]{#1}
\csname url@samestyle\endcsname
\providecommand{\newblock}{\relax}
\providecommand{\bibinfo}[2]{#2}
\providecommand{\BIBentrySTDinterwordspacing}{\spaceskip=0pt\relax}
\providecommand{\BIBentryALTinterwordstretchfactor}{4}
\providecommand{\BIBentryALTinterwordspacing}{\spaceskip=\fontdimen2\font plus
\BIBentryALTinterwordstretchfactor\fontdimen3\font minus
  \fontdimen4\font\relax}
\providecommand{\BIBforeignlanguage}[2]{{%
\expandafter\ifx\csname l@#1\endcsname\relax
\typeout{** WARNING: IEEEtran.bst: No hyphenation pattern has been}%
\typeout{** loaded for the language `#1'. Using the pattern for}%
\typeout{** the default language instead.}%
\else
\language=\csname l@#1\endcsname
\fi
#2}}
\providecommand{\BIBdecl}{\relax}
\BIBdecl

\bibitem{Jordan255}
M.~I. Jordan and T.~M. Mitchell, ``Machine learning: Trends, perspectives, and
  prospects,'' \emph{Science}, vol. 349, no. 6245, pp. 255--260, 2015.

\bibitem{10.1145/3298981}
Q.~Yang, Y.~Liu, T.~Chen, and Y.~Tong, ``Federated machine learning: Concept
  and applications,'' \emph{ACM Trans. Intell. Syst. Technol.}, vol.~10, no.~2,
  Jan. 2019.

\bibitem{8940936}
Q.~Yang, Y.~Liu, Y.~Cheng, Y.~Kang, T.~Chen, and H.~Yu, \emph{Federated
  Learning}.\hskip 1em plus 0.5em minus 0.4em\relax Cham, Switzerland:
  Springer, 2019.

\bibitem{9084352}
T.~Li, A.~K. Sahu, A.~Talwalkar, and V.~Smith, ``Federated learning:
  Challenges, methods, and future directions,'' \emph{{IEEE} Signal Process.
  Mag.}, vol.~37, no.~3, pp. 50--60, 2020.

\bibitem{pmlr-v70-guo17a}
C.~Guo, G.~Pleiss, Y.~Sun, and K.~Q. Weinberger, ``On calibration of modern
  neural networks,'' in \emph{Proc. Intl. Conf. Mach. Learning (ICML)},
  (Sydney, Australia), Aug. 2017, pp. 1321--1330.

\bibitem{lakshminarayanan2017simple}
B.~Lakshminarayanan, A.~Pritzel, and C.~Blundell, ``Simple and scalable
  predictive uncertainty estimation using deep ensembles,'' in \emph{Proc.
  Advances Neural Inf. Process. Syst.}, 2017, pp. 6402--6413.

\bibitem{Barber2012BayesianRA}
D.~Barber, \emph{Bayesian reasoning and machine learning}.\hskip 1em plus 0.5em
  minus 0.4em\relax Cambridge, U.K.: Cambridge University Press, 2012.

\bibitem{8453245}
O.~Simeone, ``A brief introduction to machine learning for engineers,''
  \emph{Found. Trends Signal Process.}, vol.~12, no. 3-4, pp. 200--431, 2018.

\bibitem{vehtari2020expectation}
A.~Vehtari, A.~Gelman, T.~Sivula, P.~Jyl{\"a}nki, D.~Tran, S.~Sahai,
  P.~Blomstedt, J.~P. Cunningham, D.~Schiminovich, and C.~P. Robert,
  ``Expectation propagation as a way of life: A framework for bayesian
  inference on partitioned data,'' \emph{J. Mach. Learn. Res.}, vol.~21, pp.
  1--53, 2020.

\bibitem{bui2018partitioned}
\BIBentryALTinterwordspacing
T.~D. Bui, C.~V. Nguyen, S.~Swaroop, and R.~E. Turner, ``Partitioned
  variational inference: A unified framework encompassing federated and
  continual learning,'' 2018. [Online]. Available:
  \url{https://arxiv.org/pdf/1811.11206.pdf}
\BIBentrySTDinterwordspacing

\bibitem{corinzia2019variational}
\BIBentryALTinterwordspacing
L.~Corinzia, A.~Beuret, and J.~M. Buhmann, ``Variational federated multi-task
  learning,'' 2019. [Online]. Available:
  \url{https://arxiv.org/pdf/1906.06268.pdf}
\BIBentrySTDinterwordspacing

\bibitem{ren2020accelerating}
J.~Ren, G.~Yu, and G.~Ding, ``Accelerating dnn training in wireless federated
  edge learning systems,'' \emph{{IEEE} J. Sel. Areas Commun.}, vol.~39, no.~1,
  pp. 219--232, 2020.

\bibitem{zhao2020federated}
Z.~Zhao, C.~Feng, H.~H. Yang, and X.~Luo, ``Federated-learning-enabled
  intelligent fog radio access networks: Fundamental theory, key techniques,
  and future trends,'' \emph{{IEEE} Wireless Commun.}, vol.~27, no.~2, pp.
  22--28, 2020.

\bibitem{li2019convergence}
X.~Li, K.~Huang, W.~Yang, S.~Wang, and Z.~Zhang, ``On the convergence of fedavg
  on non-iid data,'' in \emph{Proc. Int. Conf. Learning Representations}, 2020.

\bibitem{zhao2018federated}
\BIBentryALTinterwordspacing
Y.~Zhao, M.~Li, L.~Lai, N.~Suda, D.~Civin, and V.~Chandra, ``Federated learning
  with non-iid data,'' 2018. [Online]. Available:
  \url{https://arxiv.org/pdf/1806.00582.pdf}
\BIBentrySTDinterwordspacing

\bibitem{zhao2015stochastic}
P.~Zhao and T.~Zhang, ``Stochastic optimization with importance sampling for
  regularized loss minimization,'' in \emph{Proc. 32nd Int. Conf. Mach.
  Learn.}\hskip 1em plus 0.5em minus 0.4em\relax PMLR, 2015, pp. 1--9.

\bibitem{liu2020data}
Y.~Liu, Z.~Zeng, W.~Tang, and F.~Chen, ``Data-importance aware radio resource
  allocation: Wireless communication helps machine learning,'' \emph{{IEEE}
  Commun. Lett.}, vol.~24, no.~9, pp. 1981--1985, 2020.

\bibitem{zeng2021noise}
Z.~Zeng, Y.~Liu, W.~Tang, and F.~Chen, ``Noise is useful: Exploiting data
  diversity for edge intelligence,'' \emph{IEEE Wireless Commun. Lett.},
  vol.~10, no.~5, pp. 957--961, 2021.

\bibitem{ahn2014distributed}
S.~Ahn, B.~Shahbaba, and M.~Welling, ``Distributed stochastic gradient mcmc,''
  in \emph{Proc. Int. Conf. Mach. Learn.}, 2014, pp. 1044--1052.

\bibitem{pmlr-v115-mesquita20a}
D.~Mesquita, P.~Blomstedt, and S.~Kaski, ``Embarrassingly parallel mcmc using
  deep invertible transformations,'' in \emph{Proceedings of The 35th
  Uncertainty in Artificial Intelligence Conference}, 2020, pp. 1244--1252.

\bibitem{welling2011bayesian}
M.~Welling and Y.~W. Teh, ``Bayesian learning via stochastic gradient langevin
  dynamics,'' in \emph{Proc. 28th Int. Conf. Mach. Learn.}, 2011, pp. 681--688.

\bibitem{angelino2016patterns}
\BIBentryALTinterwordspacing
E.~Angelino, M.~J. Johnson, and R.~P. Adams, ``Patterns of scalable bayesian
  inference,'' 2016. [Online]. Available:
  \url{https://arxiv.org/pdf/1602.05221.pdf}
\BIBentrySTDinterwordspacing

\bibitem{kassab2022federated}
R.~Kassab and O.~Simeone, ``Federated generalized bayesian learning via
  distributed stein variational gradient descent,'' \emph{IEEE Trans. Signal
  Process.}, vol.~70, pp. 2180--2192, 2022.

\bibitem{10.5555/3157096.3157362}
Q.~Liu and D.~Wang, ``Stein variational gradient descent: A general purpose
  bayesian inference algorithm,'' in \emph{Proc. 30th Int. Conf. Neural Inf.
  Process. Syst.}, 2016, p. 2378–2386.

\bibitem{sattler2019sparse}
F.~Sattler, S.~Wiedemann, K.-R. M{\"u}ller, and W.~Samek, ``Sparse binary
  compression: Towards distributed deep learning with minimal communication,''
  in \emph{Proc. IEEE Int. Joint Conf. Neural Netw. (IJCNN)}, 2019, pp. 1--8.

\bibitem{lin2018deep}
Y.~Lin, S.~Han, H.~Mao, Y.~Wang, and B.~Dally, ``Deep gradient compression:
  Reducing the communication bandwidth for distributed training,'' in
  \emph{Proc. Int. Conf. Learning Representations}, 2018.

\bibitem{wang2019adaptive}
S.~Wang, T.~Tuor, T.~Salonidis, K.~K. Leung, C.~Makaya, T.~He, and K.~Chan,
  ``Adaptive federated learning in resource constrained edge computing
  systems,'' \emph{{IEEE} J. Sel. Areas Commun.}, vol.~37, no.~6, pp.
  1205--1221, 2019.

\bibitem{chen2019performance}
M.~Chen, Z.~Yang, W.~Saad, C.~Yin, H.~V. Poor, and S.~Cui, ``Performance
  optimization of federated learning over wireless networks,'' in \emph{Proc.
  IEEE Global Commun. Con.}, Waikoloa, HI, US, Dec. 2019, pp. 1--6.

\bibitem{tran2019federated}
N.~H. Tran, W.~Bao, A.~Zomaya, M.~N. Nguyen, and C.~S. Hong, ``Federated
  learning over wireless networks: Optimization model design and analysis,'' in
  \emph{Proc. IEEE Conf. Comput. Commun. (INFOCOM)}, Paris, France, Apr. 2019,
  pp. 1387--1395.

\bibitem{yang2020federated}
K.~Yang, T.~Jiang, Y.~Shi, and Z.~Ding, ``Federated learning via over-the-air
  computation,'' \emph{{IEEE} Trans. Wireless Commun.}, vol.~19, no.~3, pp.
  2022--2035, 2020.

\bibitem{zhu2019broadband}
G.~Zhu, Y.~Wang, and K.~Huang, ``Broadband analog aggregation for low-latency
  federated edge learning,'' \emph{{IEEE} Trans. Wireless Commun.}, vol.~19,
  no.~1, pp. 491--506, 2019.

\bibitem{zhu2020one}
G.~Zhu, Y.~Du, D.~G{\"u}nd{\"u}z, and K.~Huang, ``One-bit over-the-air
  aggregation for communication-efficient federated edge learning: Design and
  convergence analysis,'' \emph{{IEEE} Trans. Wireless Commun.}, 2020.

\bibitem{9107235}
D.~Liu, G.~Zhu, J.~Zhang, and K.~Huang, ``Data-importance aware user scheduling
  for communication-efficient edge machine learning,'' \emph{IEEE Trans. Cogn.
  Commun. Netw.}, vol.~7, no.~1, pp. 265--278, Jun. 2021.

\bibitem{9170917}
J.~Ren, Y.~He, D.~Wen, G.~Yu, K.~Huang, and D.~Guo, ``Scheduling for cellular
  federated edge learning with importance and channel awareness,'' \emph{{IEEE}
  Trans. Wireless Commun.}, vol.~19, no.~11, pp. 7690--7703, 2020.

\bibitem{9252927}
H.~T. Nguyen, V.~Sehwag, S.~Hosseinalipour, C.~G. Brinton, M.~Chiang, and
  H.~Vincent~Poor, ``Fast-convergent federated learning,'' \emph{{IEEE} J. Sel.
  Areas Commun.}, vol.~39, no.~1, pp. 201--218, 2021.

\bibitem{pmlr-v48-liub16}
Q.~Liu, J.~Lee, and M.~Jordan, ``A kernelized stein discrepancy for
  goodness-of-fit tests,'' in \emph{Proc. 33rd Int. Conf. Mach. Learn.}, 2016,
  pp. 276--284.

\bibitem{berlinet2011reproducing}
A.~Berlinet and C.~Thomas-Agnan, \emph{Reproducing kernel Hilbert spaces in
  probability and statistics}.\hskip 1em plus 0.5em minus 0.4em\relax New York,
  NY, USA: Springer, 2011.

\bibitem{knoblauch2019generalized}
J.~Knoblauch, J.~Jewson, and T.~Damoulas, ``Generalized variational
  inference,'' \emph{stat}, vol. 1050, p.~21, 2019.

\bibitem{gershman2012nonparametric}
S.~Gershman, M.~D. Hoffman, and D.~M. Blei, ``Nonparametric variational
  inference.'' in \emph{Proc. 29th Int. Conf. Mach. Learn.}, 2012.

\bibitem{hernandez2015probabilistic}
J.~M. Hern{\'a}ndez-Lobato and R.~Adams, ``Probabilistic backpropagation for
  scalable learning of bayesian neural networks,'' in \emph{Proc. Int. Conf.
  Mach. Learn.}, 2015, pp. 1861--1869.

\end{thebibliography}

\begin{IEEEbiography}
[{\includegraphics[width=1in,height=1.25in,clip,keepaspectratio]{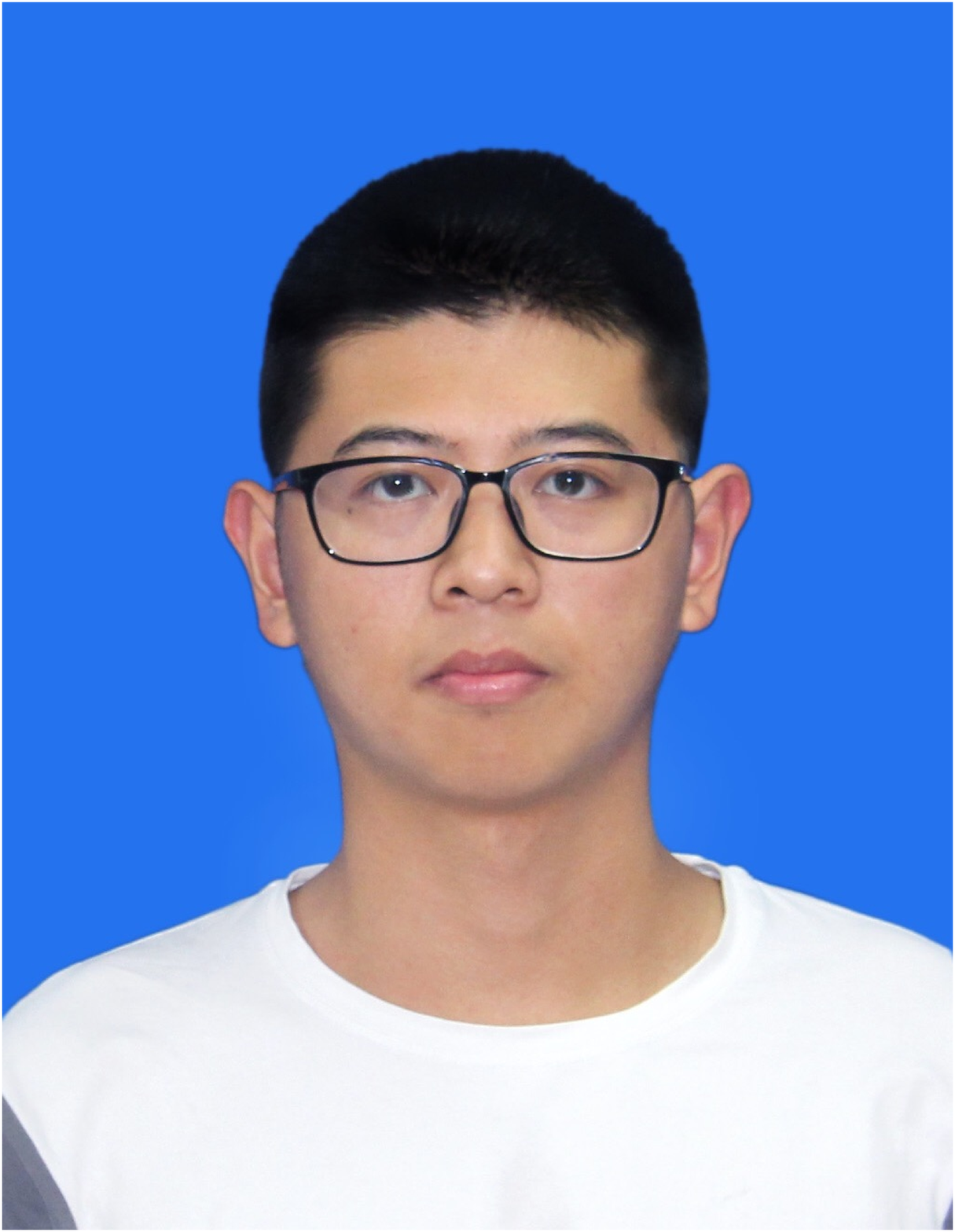}}]
{Jiarong Yang} received the B.S. degree from South China University of Technology, Guangzhou, China, in 2021. He is currently pursuing the M.S. degree with the School of Electronic and Information Engineering, South China University of Technology, Guangzhou, China. His research interests include federated learning, and Bayesian learning.
\end{IEEEbiography}
\begin{IEEEbiography}
[{\includegraphics[width=1in,height=1.25in,clip,keepaspectratio]{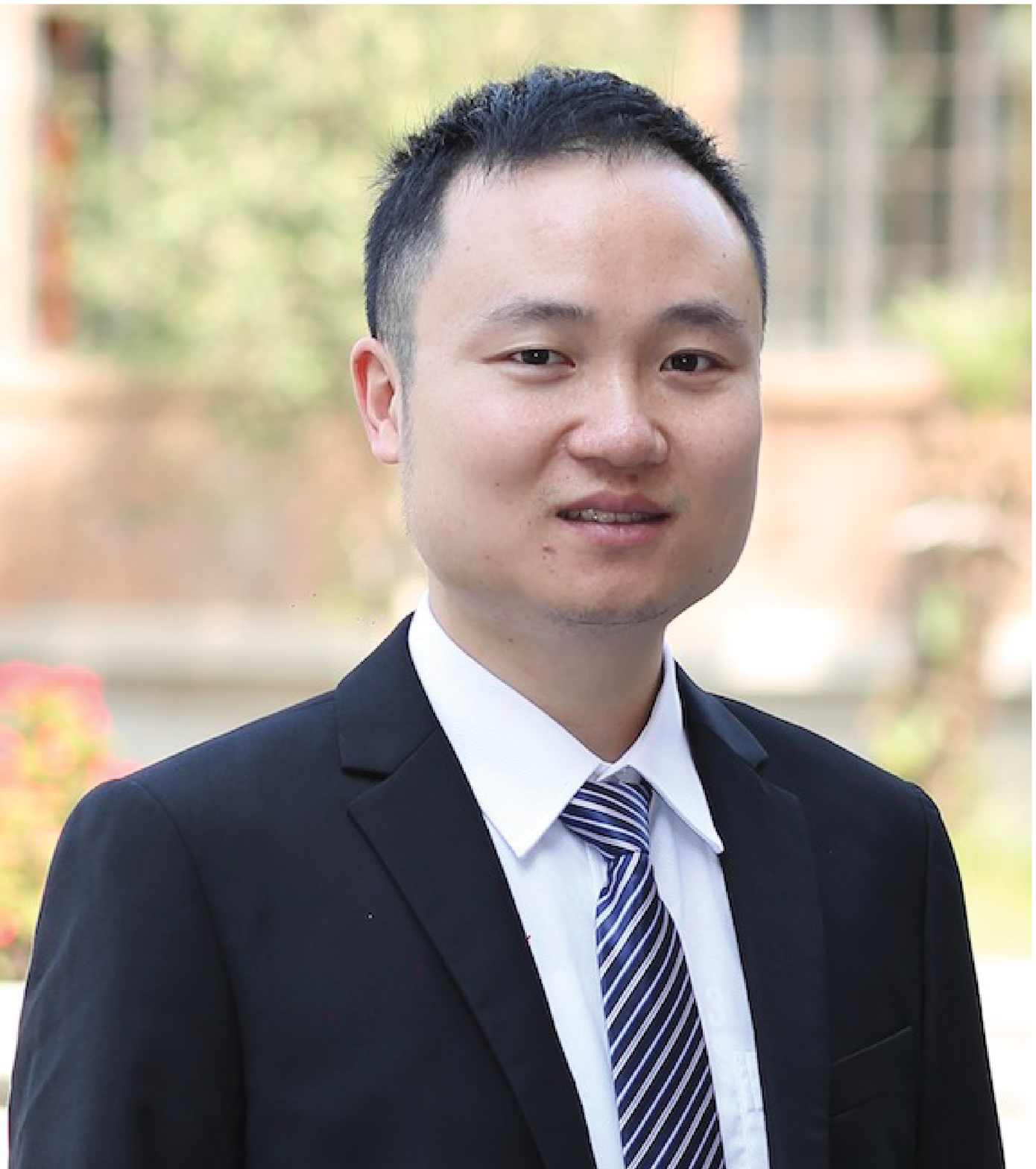}}]
{Yuan Liu}(Senior Member, IEEE) received the B.S. degree from Hunan University of Science and Technology, Xiangtan, China, in 2006; the M.S. degree from Guangdong University of Technology, Guangzhou, China, in 2009; and the Ph.D. degree from Shanghai Jiao Tong University, China, in 2013, all in electronic engineering. Since 2013, he has been with the School of Electronic and Information Engineering, South China University of Technology, Guangzhou, where he is currently an associate professor.
He serves as an editor for the \textsc{IEEE Communications Letters} and the \textsc{IEEE Access}. His research interests include 5G communications and beyond, mobile edge computation offloading, and machine learning in wireless networks.
\end{IEEEbiography}
\begin{IEEEbiography}
[{\includegraphics[width=1in,height=1.25in,clip,keepaspectratio]{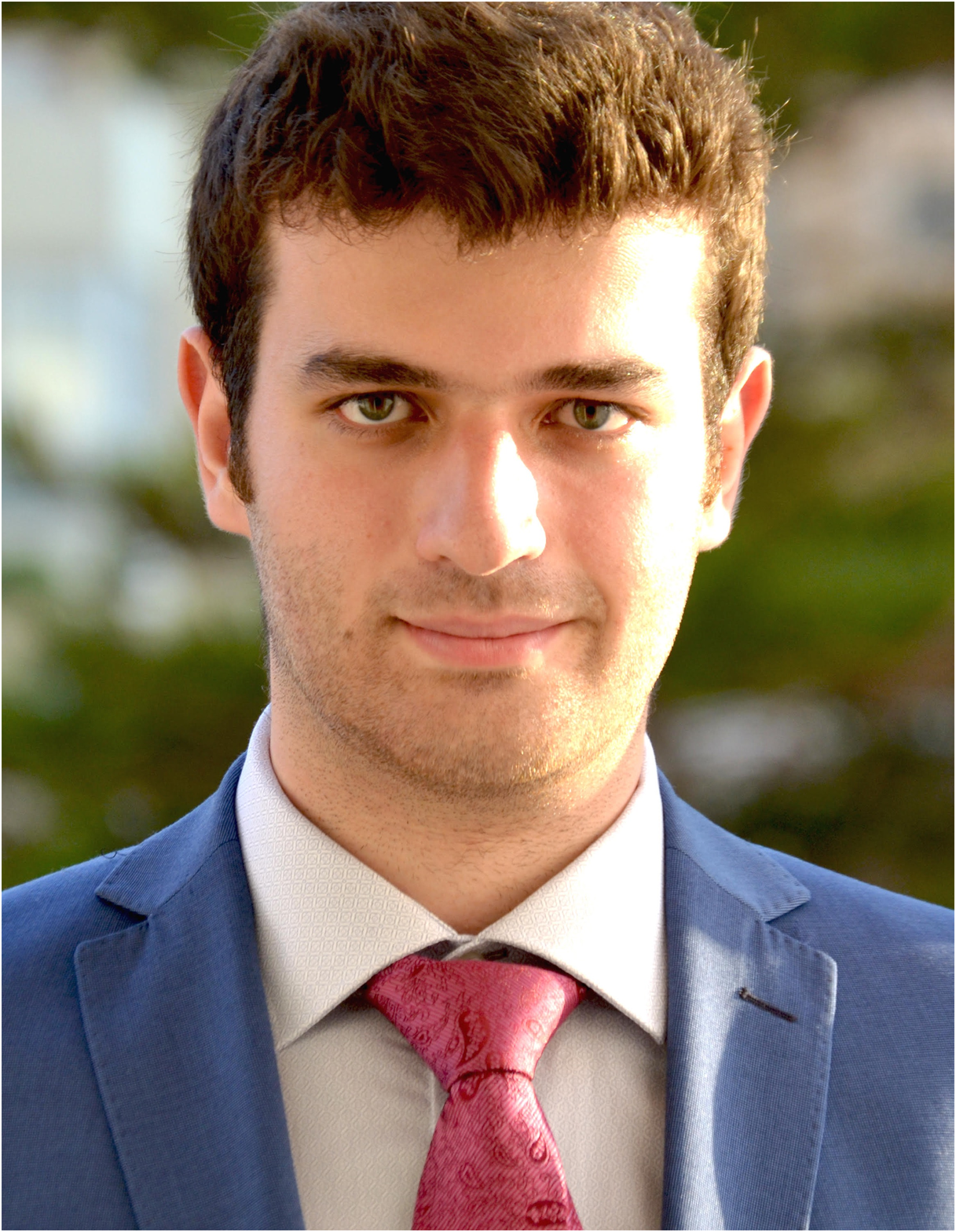}}]
{Rahif Kassab} received the two Engineering degrees from Telecom Paris, Paris, France, and Lebanese University, Beirut, Lebanon, in 2017, the M.Sc. degree in advanced communication networks jointly from École Polytechnique, France, and Telecom Paris, and the Ph.D. degree in computer science from King’s College London, London, U.K., in 2021. His research interests include communication theory, optimization, and machine learning. His industrial experience includes a six-month internship with Nokia Bell Labs and a summer internship with Huawei’s Mathematical and Algorithmic Sciences Lab, Paris, France. He received the 2018 IEEE Globecom Student Travel Grant, the Ile-de-France Masters Scholarship, from 2015 to 2017, and a Ph.D. Fellowship, from 2018 to 2020, awarded from King’s College London and funded by the European Research Council. 
\end{IEEEbiography}

\end{document}